\documentclass{article} % For LaTeX2e
\PassOptionsToPackage{numbers, compress}{natbib}

\usepackage{iclr2026_conference,times}

\usepackage{amsmath,amsfonts,bm,dsfont}

\def\eqref#1{equation~\ref{#1}}
\def\1{\bm{1}}

\def\vc{{\bm{c}}}

\def\vr{{\bm{r}}}

\def\mA{{\bm{A}}}

\def\mC{{\bm{C}}}

\def\mH{{\bm{H}}}
\def\mI{{\bm{I}}}

\def\mK{{\bm{K}}}

\def\mO{{\bm{O}}}
\def\mP{{\bm{P}}}
\def\mQ{{\bm{Q}}}
\def\mR{{\bm{R}}}
\def\mS{{\bm{S}}}

\def\mW{{\bm{W}}}
\def\mX{{\bm{X}}}
\def\mY{{\bm{Y}}}

\def\mPi{{\bm{\Pi}}}

\DeclareMathAlphabet{\mathsfit}{\encodingdefault}{\sfdefault}{m}{sl}
\SetMathAlphabet{\mathsfit}{bold}{\encodingdefault}{\sfdefault}{bx}{n}

\def\sR{{\mathbb{R}}}

\newcommand{\E}{\mathbb{E}}

\usepackage{xspace}
\usepackage[inline]{enumitem}

\newcommand{\ourmethod}{\textsc{All-In}\xspace}

\usepackage{wrapfig}
\usepackage{multirow}
\usepackage{pdfpages}
\usepackage{booktabs}
\usepackage{adjustbox}

\usepackage{amsthm}
\usepackage{thmtools}
\usepackage{thm-restate}

\theoremstyle{plain}

\theoremstyle{definition}

\theoremstyle{remark}

\usepackage{hyperref}
\usepackage{url}

\usepackage[utf8]{inputenc} % allow utf-8 input
\usepackage[T1]{fontenc}    % use 8-bit T1 fonts
\usepackage{hyperref}       % hyperlinks
\usepackage{url}            % simple URL typesetting
\usepackage{booktabs}       % professional-quality tables
\usepackage{amsfonts}       % blackboard math symbols
\usepackage{nicefrac}       % compact symbols for 1/2, etc.
\usepackage{microtype}      % microtypography
\usepackage{xcolor}         % colors
\hypersetup{
    colorlinks,
    linkcolor={red!50!black},
    citecolor={blue!50!black},
    urlcolor={blue!80!black}
}

\usepackage{cleveref}
\usepackage{subcaption}

\iclrfinalcopy

\title{Bridging Input Feature Spaces Towards \\ Graph Foundation Models}

\author{Moshe Eliasof \\ University of Cambridge \\
Ben-Gurion University of the Negev \\
\texttt{me532@cam.ac.uk} \\
\And
Krishna Sri Ipsit Mantri \\
Purdue University\\
\texttt{mantrik@purdue.edu}\\
\And
Beatrice Bevilacqua \\
Purdue University\\
\texttt{bbevilac@purdue.edu}\\
\And
Bruno Ribeiro \\
Purdue University\\
\texttt{ribeirob@purdue.edu}\\
\And
Carola-Bibiane Sch\"onlieb \\
University of Cambridge\\
\texttt{cbs31@cam.ac.uk}
}

\begin{document}

\maketitle
% \lhead{Preprint}

\begin{abstract}
Unlike vision and language domains, graph learning lacks a shared input space, as input features differ across graph datasets not only in semantics, but also in value ranges and dimensionality. This misalignment prevents graph models from generalizing across datasets, limiting their use as foundation models.
In this work, we propose \ourmethod, a simple and theoretically grounded method that enables transferability across datasets with different input features. Our approach projects node features into a shared random space and constructs representations via covariance-based statistics, thus eliminating dependence on the original feature space. 
We show that the computed node-covariance operators and the resulting node representations are invariant in distribution to permutations of the input features. We further demonstrate that the expected operator exhibits invariance to general orthogonal transformations of the input features.
Empirically, \ourmethod achieves strong performance across diverse node- and graph-level tasks on unseen datasets with new input features, without requiring architecture changes or retraining. These results point to a promising direction for input-agnostic, transferable graph models.
\end{abstract}
\section{Introduction}
\label{sec:intro}

Foundation models have shown remarkable success in domains such as language and vision, where large-scale pretraining enables strong performance across a wide range of downstream tasks. A similar goal has emerged for graph learning: to develop graph foundation models that generalize across tasks, domains, and datasets~\citep{mao2024graph}. However, a key obstacle in this direction is the lack of transferability across graphs, as knowledge learned from one graph is often difficult to apply to another due to fundamental differences in their structure and, critically, their input features.

Unlike vision or language data, graph datasets typically do not share a common input space. Node features often differ significantly not only in distribution and semantics but also in dimensionality from one graph to another. Furthermore, graphs themselves may vary in size, sparsity, and topological patterns.
These mismatches break many of the assumptions that underlie successful generalization in other domains, making it difficult to define a common representation space or pretraining strategy.

Existing approaches to graph foundation models fall into two broad categories. The first integrates LLMs by serializing graph data into text or designing prompt-based mechanisms~\citep{liu2024one,zhao2023graphtext,chen2024exploring,fatemi2024talk,perozzi2024letgraphtalk,chen2024llaga,zhao2023graphtext,he2024unigraph,huang2023prodigy,tang2024graphgpt, kim2024revisiting, zhao2024oneforall, gong2024self, sun2022gppt, sun2023all}, leveraging LLM capabilities but often discarding fine-grained graph properties. The second direction aims to explicitly align or adapt feature spaces across datasets using techniques like input projections~\citep{xia2024anygraph,yu2024textfree,zhao2024oneforall}, specialized encoders~\citep{lachi2024graphfm}, structuralization~\citep{frasca2024cross}, or order statistics~\citep{shen2024zero}. However, these methods often remain specialized to particular settings or tasks, or may require careful adaptation to new scenarios. %

\begin{figure}[t]
    \centering
    \begin{subfigure}[b]{0.5\textwidth}
        \includegraphics[width=\linewidth]{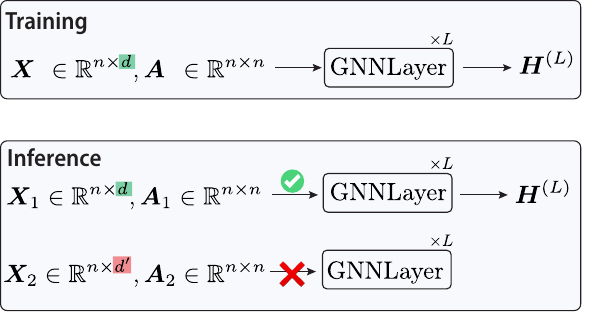 }
        \caption{Motivation: The challenge of varying input features}
    \end{subfigure}
    \begin{subfigure}[b]{0.49\textwidth}
        \centering
        \includegraphics[width=\linewidth]{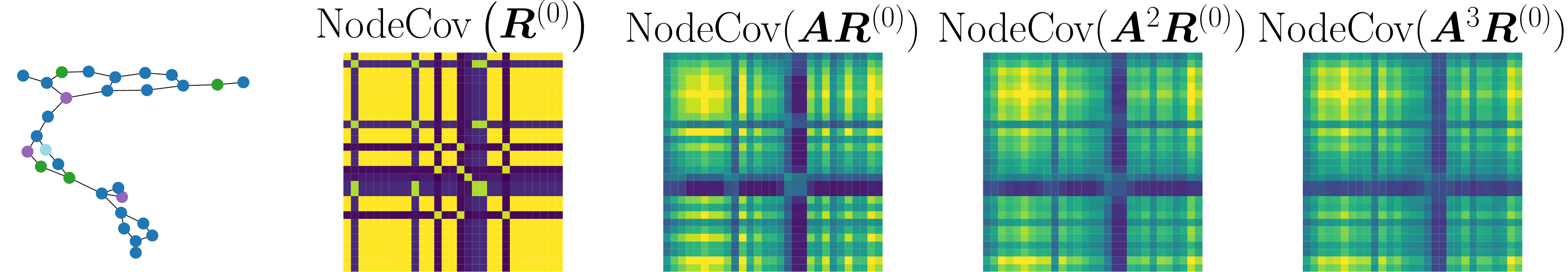}
        \includegraphics[width=\linewidth]{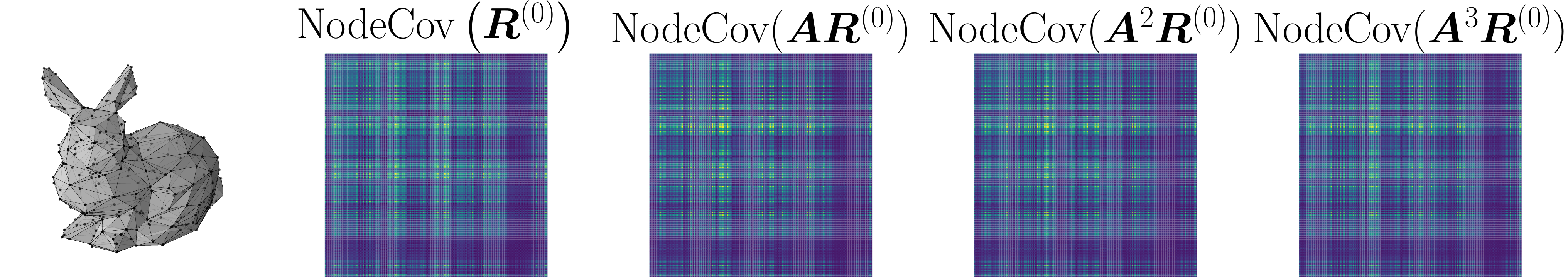}
        \includegraphics[width=\linewidth]{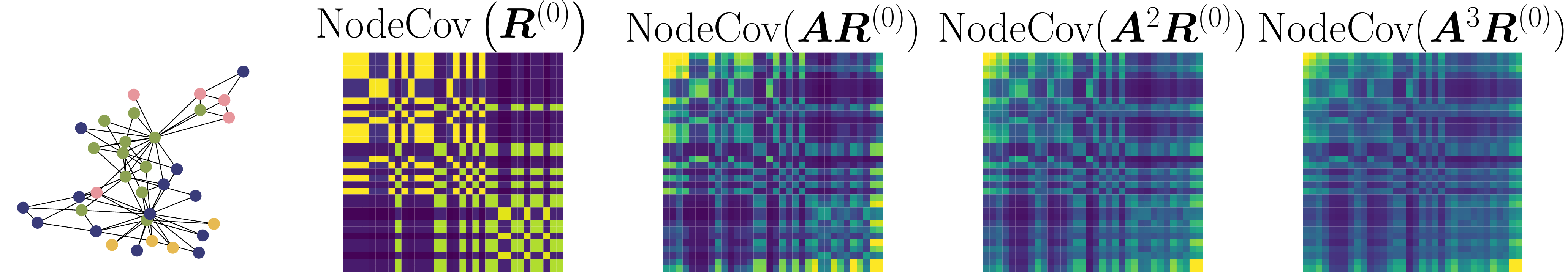}
        \caption{\ourmethod consistent node-covariance operators}
  \end{subfigure}
    \caption{\textbf{Addressing feature heterogeneity with \ourmethod's node-covariance operators.} (a) When a GNN is trained on graph data with node features $\mX$ of dimension $d$, it cannot be directly applied on graphs with features of a different dimensionality $d'$. (b) \ourmethod computes $n \times n$ node-covariance operators, capturing node similarities, providing a common space that is independent of the original, heterogeneous, feature spaces. Different node colors indicate distinct node features.
  }
    \label{fig:motivation}
    
\end{figure}
In this work, we propose a novel approach, grounded in statistical principles, to overcome input feature heterogeneity (\Cref{fig:motivation}). Our method first projects potentially disparate node features into a common, high-dimensional space using a stochastic projection matrix. We then leverage second-order statistics within this space using covariance operators. Specifically, we model feature dimensions as independent and identically distributed samples from an unknown distribution over the nodes, and compute the empirical node-covariance matrix based on these projected representations. This matrix captures pairwise node similarities based on how their projected features co-vary, providing a representation inherently robust to changes in feature semantics, value, and dimensionality.

We introduce \ourmethod (All Input spaces), a graph learning framework built upon this principle. 
Instead of directly processing raw node features in downstream layers, \ourmethod utilizes the computed stochastic node-covariance matrix (and its higher-order variants), as shown in \Cref{fig:motivation}, as a graph operator within a graph neural network (GNN). This node-covariance matrix captures interactions between nodes, specifically, how similar two nodes are in terms of their feature activations across the feature dimensions. Our theoretical analysis reveals significant robustness properties: \begin{enumerate*}[label=(\alph*)]
\item The computed operators and, critically, the resulting node representations throughout the GNN are invariant in distribution to arbitrary permutations of the original input features;
\item The expected operator is invariant to general orthogonal transformations (basis changes) of the input features;
\item The overall method is inherently insensitive to dimensional mismatches across datasets.
\end{enumerate*}
We further identify qualitative conditions under which covariance-based representations retain task-relevant information and enable transfer across datasets with different input features.

Our empirical results confirm the efficacy of this approach: \ourmethod achieves strong transfer performance to new datasets with new input features across diverse node- and graph-level tasks. As a result, \ourmethod offers a promising approach toward the development of graph foundation models.

\section{Related Work}
\label{sec:related_work}

\textbf{Graph Foundation Models (GFMs).} GFMs aim to learn representations that generalize across datasets and tasks, but achieving robust generalization remains challenging, especially when node features change. Some approaches integrate LLMs by converting graphs to text or embedding features through prompt-based
designs~\citep{liu2024one,zhao2023graphtext,chen2024exploring,fatemi2024talk,perozzi2024letgraphtalk,chen2024llaga,zhao2023graphtext,he2024unigraph,huang2023prodigy,tang2024graphgpt, kim2024revisiting, zhao2024oneforall, gong2024self, sun2022gppt, sun2023all}, or by generating or augmenting graphs with LLM guidance before training a graph encoder~\citep{xia2024opengraph}, but this can lead to loss of structural details. Text-attributed GFMs further learn transferable vocabularies or automatically search architectures on such LLM-derived features~\citep{wang2024gft, chen2025autogfm}, which improves transfer within TAGs but does not directly handle non-textual node attributes. Other works align feature spaces through projections~\citep{xia2024anygraph,yu2024textfree,zhao2024oneforall, fang2023universal} and multi-domain feature or structure aligners with prompts or mixtures-of-experts~\citep{yu2025samgpt, yuan2025howbridge}, perceiver-based encoders~\citep{lachi2024graphfm}, computing analytical solutions (in the case of node classification)~\citep{zhao2024graphany}, encoding features into the graph structure~\citep{frasca2024cross, galkin2024towards, wang2024towards, franks2025towards} or learning shared structural vocabularies in Riemannian spaces~\citep{sun2025riemanngfm}, or encoding feature relationships~\citep{shen2024zero}. While these methods advance GFM capabilities, they often require task-specific adaptations, leaving a gap for truly input-space-agnostic solutions. \ourmethod offers a distinct path: it creates transferable representations by processing arbitrary input features through stochastic projections and node-covariance operators, enabling frozen-encoder transfer without task- or domain-specific prompts or architectural changes.

\textbf{Structural and Positional Encodings.} Efforts to create universal graph representations include transferable structural and positional encodings (SPEs)~\citep{rampavsek2022recipe,canturk2024graph,chen2025towards, kim2024revisiting}. SPEs aim to capture graph topology in a feature-agnostic manner, often within Graph Transformers or GNNs. While such SPEs can complement node features, \ourmethod directly addresses the challenge of heterogeneous node features themselves, transforming them into a robust, transferable format using their covariance structure, irrespective of any additional SPEs.

\textbf{Covariance networks.}
Covariance matrices have also informed the design of neural networks. For instance, coVariance Neural Networks (VNNs)~\citep{saurabh2022covariance} process $d \times d$ sample covariance matrices, with $d$ the input feature dimension, which describe feature inter-correlations, offering benefits like stability to varying sample sizes and inspiring extensions for fairness~\citep{cavallo2025fair} and sparsity~\citep{cavallo2024sparse}. Other related efforts focus on transferring principal components derived from data covariance~\citep{hendy2024tlpca}.
While these methods analyze relationships between features using sample covariance matrices, \ourmethod constructs an $n \times n$ node-covariance matrix, with $n$ number of nodes. 
This operator quantifies similarities between pairs of nodes based on how their (randomly projected) features co-vary across dimensions. This distinct formulation is tailored to building transferable representations from graphs with heterogeneous node features, addressing a challenge different from that targeted by the aforementioned approaches.

\section{Method}
\label{sec:method}

Our method, \ourmethod, replaces dataset-specific raw node features with covariance-based operators that are better suited for generalization across input feature spaces. The approach comprises three main stages: (1) Random Feature Projection to map input features to a shared space, (2) Node-Covariance Operator computation to capture robust node similarities, and (3) Operator-based Propagation to learn transferable node representations. An overview of \ourmethod can be found in \Cref{fig:method}.

\begin{figure}[t]
    \centering
    \includegraphics[width=\linewidth]{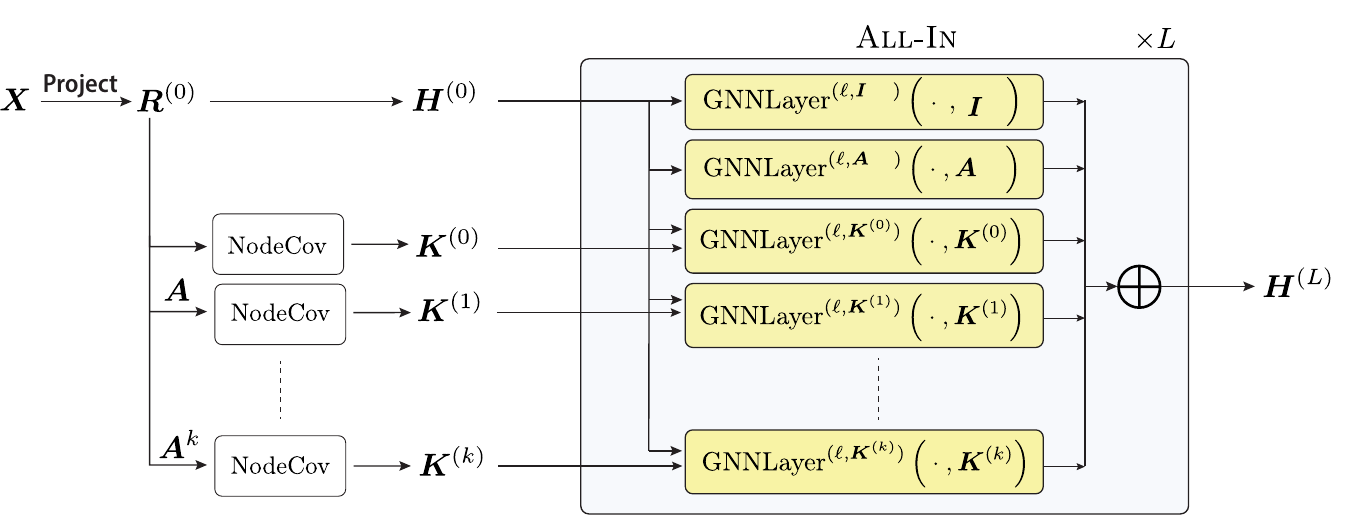}
    \caption{\textbf{The \ourmethod Architecture.} Input node features $\mX$ are first randomly projected into $\mR^{(0)}$. This $\mR^{(0)}$ serves as  initial node representations $\mH^{(0)}$. Concurrently, $\mR^{(0)}$ and its propagated versions (e.g., $\mR^{(p)} = \mA^p \mR^{(0)}$) are used to compute a set of node-covariance matrices $\{\mK^{(p)}\}_{p=0}^k$ capturing diverse orders of feature-based node similarities. These matrices are used as operators within different GNN (sub-)layers, whose outputs are concatenated to form the updated node representation.}
    \label{fig:method}
\end{figure}

\paragraph{Random Feature Projections.} Given a graph with $n$ nodes and node feature matrix $\mX \in \mathbb{R}^{n \times d}$, where the input dimension $d$ may vary across graph datasets, we first apply a random linear transformation to project the features into a unified fixed-dimensional space $h$ that is shared across datasets: 
\begin{align}\label{eq:projectednode}
    \mR^{(0)} = \mX \mC , \qquad \text{with vec}(\mC) \sim \mathcal{N}(\mathbf{0}, \mI_{d h})\text{ sampled at each forward pass},
\end{align}
that is, $\mC \in \mathbb{R}^{d \times h}$ is an isotropic Gaussian random weight matrix sampled independently at each forward pass. This step is key to ensuring that our approach is invariant (in distribution) to feature permutations, as we discuss in \Cref{sec:theory}.

\paragraph{Node-Covariance Operators.} We treat each column of $\mR^{(0)}$ as an i.i.d. signal over the nodes and compute the node-covariance matrix to capture second-order relationships (node similarities) based on feature co-variation across the latent dimensions:
\begin{align}\label{eq:cov}
    \mK^{(0)} = \text{NodeCov}(\mR^{(0)}) = \frac{1}{h}\mR^{(0)}_c {\mR^{(0)}_c}^T \in \mathbb{R}^{n \times n},
\end{align}  
where $\mR^{(0)}_c \in \mathbb{R}^{n \times h}$ is the centered projected feature matrix defined by
$\mR^{(0)}_c = \mR^{(0)} - \mathbf{1}_n \bar{\vr}$ with $\bar{\vr} = \frac{1}{n} \sum_i^n \mR^{(0)}_i \in \mathbb{R}^{1 \times h}
$ the  empirical mean of the projected node features, and $\mathbf{1}_n \in \mathbb{R}^{n \times 1}$ the all-ones vector. 
This centering operation is equivalent to pre-multiplying by the geometric centering matrix $\mPi_c = \mI_n - \frac{1}{n}\mathbf{1}_n\mathbf{1}_n^T$, i.e., $\mR^{(0)}_c = \mPi_c \mR^{(0)}$.
The resulting $\mK^{(0)}$ is an $n \times n$ matrix reflecting node similarities in the projected feature space. An interesting property is that if we consider two nodes $u$ and $v$ with feature vectors $\mX_{u}$ and $\mX_{v} = -\mX_{u}$, then their auto-covariance terms in $\mK^{(p)}$ coincide, but their rows $\mK^{(p)}_{u}$ and $\mK^{(p)}_{v}$ differ in the cross-covariance entries with other nodes because their signs flip, so message passing based on $\mK^{(p)}$ can still distinguish them.

To integrate structural information with feature similarities, we compute higher-order covariance matrices based on propagated features. Specifically, for each $p=1,2,\dots,k$, we first perform message passing on the initial projected features $\mR^{(0)}$ using the graph's adjacency matrix $\mA$: 
\begin{align*}
    \mR^{(p)} = \mA^p \mR^{(0)}.
\end{align*}  
Then, we compute the covariance matrix on these propagated, centered features $\mR^{(p)}_c = \mPi_c \mR^{(p)}$:
\begin{align}\label{eq:covp}
    \mK^{(p)} = \text{NodeCov}(\mR^{(p)}) = \frac{1}{h} \mR^{(p)}_c {\mR^{(p)}_c}^T  \in \mathbb{R}^{n \times n}.
\end{align}
The operator $\mK^{(p)}$ captures node similarities based on features aggregated from neighborhoods up to $p$ hops away, thus encoding increasingly global structural context in the graph. 

\paragraph{Node Representations.}
We collect a set of graph operators, which includes the identity matrix $\mI$, the adjacency matrix $\mA$, and the computed node-covariance matrices $\mathcal{K} = \{  \mK^{(p)}\}_{p=0}^k$:
\begin{align}\label{eq:operators}
    \mathcal{O}=\{\mI,\mA,\mK^{(0)},\mK^{(1)}, \ldots, \mK^{(k)}\}.
\end{align}
Instead of using the original node features, we rely on the random projections $\mR^{(0)}$, potentially augmented with structural encodings, such as random-walk encodings~\citep{dwivedi2022graph}. 
That is, we let $\mH^{(0)}$ be
\begin{align}\label{eq:h0}
    \mH^{(0)} = \mR^{(0)} \oplus \mS
\end{align}
where $\oplus$ indicates concatenation and $\mS \in \sR^{n \times h_s}$ is a structural encoding matrix. We note that, although the node-covariance operators $\mK^{(p)}$ capture second-order statistics, ALL-IN maintains first-order information: the projected features $\mR^{(0)}$ are used directly as part of the initial node representations $\mH^{(0)}$ in \Cref{eq:h0}. For example, if two nodes $u$ and $v$ have feature vectors $\mX_{u}$ and $\mX_{v} = -\mX_{u}$, then their projected features satisfy $\mR^{(0)}_{u} = \mX_{u}\mC$ and $\mR^{(0)}_{v} = -\mX_{u}\mC$, so they are distinguishable in $\mH^{(0)}$.

At each layer $\ell = 1, \ldots, L$, we propagate the current node representations using every operator $\mO \in \mathcal{O}$, and concatenate the outputs to obtain the updated representations:
\begin{align}\label{eq:propagation}
    \mH^{(\ell)} = \bigoplus_{\mO \in \mathcal{O}} \text{GNNLayer}^{(\ell, \mO)}(\mH^{(\ell - 1)}, \mO ),
\end{align}
where $\text{GNNLayer}^{(\ell, \mO)}$ is the GNN layer associated with operator $\mO \in \mathcal{O}$ in layer $\ell$, taking as input $\mH^{(\ell - 1)}$ and performing message passing using $\mO$ as the operator, using learnable weights $\mW^{(\ell, \mO)} \in \sR^{h^{(\ell - 1)} \times h^{(\ell)}}$ and $h^{(0)} = h + h_s$.

In the presence of edge features, which, similarly to node features, may vary across datasets, we employ an analogous strategy. Specifically, we first project the edge features into a fixed-dimensional space using an isotropic Gaussian random weight matrix, yielding edge representations that are independent of feature dimensionality. 
Then, we aggregate these projected edge features at the node level (e.g., by averaging features of incoming edges for each node) to obtain node-level representations $\mR^{(0)}_\text{edge}$ derived from edges. We then compute node-covariance matrices $\mK^{(p)}_\text{edge}$ based on these aggregated (and potentially propagated) features, similar to \Cref{eq:covp}. Finally, we add these edge-derived covariance operators $\mathcal{K}_\text{edge} =\{\mK^{(0)}_\text{edge}, \ldots, \mK^{(k)}_\text{edge}\}$ to the operator set $\mathcal{O}$ (\Cref{eq:operators}). This allows the model (\Cref{eq:propagation}) to incorporate edge information while remaining compatible across datasets with differing edge feature spaces.

\section{Theoretical Insights}\label{sec:theory}

This section establishes the theoretical foundations underpinning the ability of \ourmethod to handle heterogeneous input features and enable generalization across datasets.  A core contribution is proving the method's robustness to variations in feature representation.
We first demonstrate that the node-covariance operators and the resulting node representations are invariant \emph{in distribution} to arbitrary permutations of the input features, providing robustness to feature re-ordering. We then show that the \emph{expected node-covariance operator} is invariant to general orthogonal transformations, ensuring robustness to the choice of orthonormal basis (\Cref{sec:invariance}). Building on these properties, we validate the stochastic training procedure using Jensen's inequality under standard convexity assumptions (\Cref{sec:objective}). Finally, we discuss conditions supporting transferability, analyzing scenarios where the operator remains stable across graphs with differing feature distributions and proving its consistency for large projection dimensions (\Cref{sec:conditions}). All proofs are provided in \Cref{app:theory}.

%%%%%%%%%%%%%%%
\subsection{Invariance to Feature Space Transformations}\label{sec:invariance}

A primary obstacle to cross-dataset transfer is the lack of feature standardization, leading to arbitrary differences in feature ordering and basis choice across datasets. Our approach, centered on node-covariance after random projection, inherently addresses these issues through invariance properties.
First, the use of random isotropic Gaussian projections renders the process statistically insensitive to the order of input features. We formalize this by showing that the distribution of the projected feature matrix remains unchanged when the original features are permuted.

\begin{restatable}[Distributional Invariance of Projected Features to Feature Permutation]{proposition}{distributional}\label{prop:distributional_invariance}
Let $\mX \in \mathbb{R}^{n \times d}$ be node features, $\mP \in \mathbb{R}^{d \times d}$ be any permutation matrix, and $h$ be the projection dimension. Let $\mC \in \mathbb{R}^{d \times h}$ be an isotropic Gaussian random matrix (i.e., $\text{vec}(\mC) \sim \mathcal{N}(\mathbf{0}, \mI_{d h})$. Define the projected features as $\mR^{(0)} = \mX\mC$ and the features projected after permutation as $\bar{\mR}^{(0)} = (\mX\mP)\mC$. Then $\mR^{(0)}$ and $\bar{\mR}^{(0)}$ are equal in distribution:
$\mR^{(0)} \stackrel{d}{=} \bar{\mR}^{(0)}$.
\end{restatable}

In essence, \Cref{prop:distributional_invariance} establishes that random projections effectively ``mix'' features, rendering their original ordering statistically irrelevant after projection. More importantly, the permutation invariance is characterized {\em in distribution}, rather than pointwise: for a fixed random projection $\mC$, the features in $\mR^{(0)}$ retain sensitivity to input permutations, thereby enabling a neural network to better capture the relationships between node features and topology.

To illustrate this concept, consider three nodes $u, v, w \in V$ with features $\mX_u = (0,1)$, $\mX_v = (0,1)$, and $\mX_w = (1,0)$. Under strict (pointwise) permutation invariance, the embeddings of all nodes would be equivalent, obscuring the key distinction that $u$ and $v$ share identical features, whereas $w$ has a different feature. In contrast, distributional invariance ensures that the distributions of $\mR^{(0)}_u$, $\mR^{(0)}_v$, and $\mR^{(0)}_w$ are identical, yet individual forward passes yield different outcomes: given $\mC$, we have $\mR^{(0)}_u = \mR^{(0)}_v \neq \mR^{(0)}_w$. This property preserves the model's ability to distinguish between nodes $u$ and $v$ (which share the same features) and node $w$ (which has a different feature), while maintaining symmetry in the model's statistical behavior, thus striking a balance between permutation invariance and expressive power.

Next, we show that the NodeCov operators applied to the sequence $\{\mR^{(p)}\}_{p=0}^k$ (as defined in \Cref{eq:covp}) yield features that are also distributionally invariant.

\begin{restatable}[Distributional Invariance of Node Covariance Operators to Feature Permutation]{corollary}{distributionalnodecov}\label{cor:distributional_invariance_K}
Let $\mX \in \mathbb{R}^{n \times d}$ be node features, and $\mP \in \mathbb{R}^{d \times d}$ be any permutation matrix. Let $\mR^{(0)} = \mX\mC$ be the initial projected features.
Let $\mathcal{K} = \{\mK^{(p)}\}_{p=0}^k$ be the set of node-covariance operators, where $\mK^{(p)} = \text{NodeCov}(\mA^p \mR^{(0)})$ is computed using the deterministic function NodeCov (\Cref{eq:covp}), and $\mA$ is the adjacency matrix.
It follows directly from the distributional invariance of $\mR^{(0)}$ that the entire set of operators $\mathcal{K}$ is also invariant in distribution to permutations of the input features $\mX$. That is, if $\mathcal{\bar{K}}$ is the set of operators computed using $\mX\mP$ instead of $\mX$, then $\mathcal{K} \stackrel{d}{=} \mathcal{\bar{K}}$.
\end{restatable}

The significance of \Cref{prop:distributional_invariance} and \Cref{cor:distributional_invariance_K} is substantial: it guarantees that the complete statistical behavior of $\mR^{(0)}$ and the operators $\mK^{(p)}$ central to \ourmethod is fundamentally robust to arbitrary input feature ordering, directly addressing a key source of heterogeneity across graph datasets. This distributional invariance also extends to the hidden representations $\mH^{(\ell)}$, for all $\ell= 1 \ldots L$ derived from these operators, as shown in \Cref{thm:hidden_invariance} in \Cref{app:theory}.

The stochastic projection matrix $\mC$ plays a critical role beyond enabling the distributionally invariance properties discussed earlier; its use is intrinsically linked to the expressive capability of the learning framework. Training with node-covariance operators $\text{NodeCov}(\mR^{(0)})$ derived from these stochastic projections offers advantages over relying on a single, deterministically computed covariance operator, such as $\text{NodeCov}(\mX)$. While $\text{NodeCov}(\mX)$ provides a stable, pointwise feature-permutation invariant view of node similarities, it can obscure subtle but important distinctions between nodes. In contrast, individual stochastic realizations $\text{NodeCov}(\mR^{(0)}) = \text{NodeCov}(\mX\mC)$ (for a specific $\mC$) can preserve these finer-grained distinctions, providing richer and more varied signals to the GNN. \Cref{thm:expressivity} formalizes this concept by demonstrating that there exist instances where the stochastic operator $\text{NodeCov}(\mX\mC)$ can distinguish nodes that the deterministic operator $\text{NodeCov}(\mX)$ cannot.

\begin{restatable}[Distinguishability through $\mC$]{theorem}{expressivity}\label{thm:expressivity}
There exist node features $\mX \in \mathbb{R}^{n \times d}$, nodes $u, v \in V$ with $\mX_u \neq \mX_v$ such that $\text{NodeCov}(\mX)$ makes $u$, $v$ indistinguishable (automorphic), but $\text{NodeCov}(\mX\mC)$ (for a.s.\ all $\mC$) makes $u$, $v$ distinguishable (not automorphic).
\end{restatable}

Finally, while distributional invariance covers permutations, analyzing the expected operator reveals broader robustness to basis changes and identifies the structure captured on average, as we show next.

\begin{restatable}[Expected Invariance to Orthogonal Transformations]{theorem}{orthogonaltransformation}\label{thm:orthogonal_invariance}
Let $\mX \in \mathbb{R}^{n \times d}$ be node features, $\mQ \in \mathbb{R}^{d \times d}$ be an orthogonal matrix, and $h$ be the projection dimension. Consider a random projection matrix $\mC \in \mathbb{R}^{d \times h}$ with $\text{vec}(\mC) \sim \mathcal{N}(\mathbf{0}, \mI_{d h})$.
Let $\text{NodeCov}(\mR^{(0)}) = \frac{1}{h}(\mPi_c \mR^{(0)})(\mPi_c \mR^{(0)})^T$ be the Node Covariance operator (\Cref{eq:cov}), where $\mPi_c = \mI_n - \frac{1}{n}\mathbf{1}_n\mathbf{1}_n^T$ is the centering matrix.
Then, the expected Node Covariance computed from the stochastically projected features is invariant to the orthogonal transformation $\mQ$:
\begin{align}
\E_{\mC}[\text{NodeCov}(\mX \mQ \mC)] = \E_{\mC}[\text{NodeCov}(\mX\mC)] = \mPi_c \mX \mX^T \mPi_c
\end{align}
where the expectation $\E_{\mC}[\cdot]$ is over the random sampling of $\mC$, and $\mPi_c \mX \mX^T \mPi_c$ is the Gram matrix of the centered original features.
\end{restatable}

\Cref{thm:orthogonal_invariance} demonstrates that the expected operator is agnostic to any choice of orthonormal basis (rotations, reflections, permutations) for the input features. Furthermore, identifying this stable expectation as the Gram matrix of centered original features ($\mPi_c \mX \mX^T \mPi_c$) reveals that \ourmethod, on average, recovers intrinsic, basis-invariant pairwise node similarities directly reflecting the original data structure, irrespective of the specific random projection used.

\subsection{Training Objective Upper Bound}\label{sec:objective}

\ourmethod computes the feature projection $\mR^{(0)}$ and node-covariance operator $\mK^{(0)} = \text{NodeCov}(\mX\mC)$ using a stochastic projection matrix $\mC$ sampled in each forward pass. We now validate this practical training approach by showing its connection to performance on the stable, expected final representation $\mathbb{E}_{\mC}[\mH^{(L)}]$, assuming common convexity conditions for the final prediction layer. 

\begin{restatable}[Loss Upper Bound]{theorem}{lossjensen}\label{thm:loss_jensen}
Let $\mH^{(L)} \in \mathbb{R}^{n \times h^{(L)}}$ be the final node representations computed by \ourmethod, dependent on the initial random projection $\mC$. Let $\phi: \mathbb{R}^{n \times h^{(L)}} \to \mathbb{R}^{n \times t} $ be the final prediction layer, and let $\mathcal{L}(\cdot, \mY)$ be the loss function comparing predictions to ground truth labels $\mY$.
Assume that the composite function $f(\mH^{(L)}) = \mathcal{L}(\phi(\mH^{(L)}), \mY)$ is convex with respect to the final node representations $\mH^{(L)}$.
Then, our stochastic optimization objective provides an upper bound for the loss of the expected representation: 
\begin{align} \label{eq:jensen_loss}
\underbrace{\mathcal{L}(\phi(\mathbb{E}_{\mC}[\mH^{(L)}]), \mY)}_{\text{Loss of Expected Representation}} \leq \underbrace{\mathbb{E}_{\mC}[\mathcal{L}(\phi(\mH^{(L)}), \mY)]}_{\text{Expected Loss (Training Objective)}}
\end{align}
where the expectation $\mathbb{E}_{\mC}[\cdot]$ is taken over the random projection matrix $\mC$.
\end{restatable}
This holds, for instance, if $\phi$ is a linear map or linear plus softmax, and $\mathcal{L}$ is cross-entropy or mean squared error. \Cref{thm:loss_jensen} provides theoretical support for training with stochastic projections. \Cref{eq:jensen_loss} establishes that the expected loss minimized during training (RHS) serves as an upper bound for the loss evaluated on the stable, expected final representation (LHS). Thus, minimizing the empirical average loss (approximating the RHS) acts as a theoretically sound surrogate objective, implicitly minimizing the loss associated with the expected representation, validating our stochastic approach.

\subsection{Conditions for Transferability and Operator Consistency}\label{sec:conditions}

Beyond invariance, achieving transfer across graphs with fundamentally different feature distributions ($\mX^{(1)}, \mX^{(2)}$ for graphs $G_1$, $G_2$) relies on the stability of the underlying structure captured by the expected operator, $\mathbb{E}_{\mC}[\mK^{(0)}] = \Pi_c \mX \mX^T \Pi_c$. We posit that such stability can arise when graphs share intrinsic properties.
Plausible scenarios where such stability in the expected operator might arise include graphs exhibiting similar relational structures tied to node features (e.g., comparable label homophily if features reflect labels), originating from a shared underlying generative process (e.g., common SBM or graphon influencing features), or possessing similar distributions of node roles (e.g., hubs, bridges) if features are role-informative. In these cases, even if the specific feature realizations differ, the resulting $\mPi_c \mX^{(i)} (\mX^{(i)})^T \mPi_c$ matrices may capture analogous relational structures.

For this potential transfer to be practically realized, the stochastic operator $\mK_h^{(0)}$ computed using a finite projection dimension $h$ must reliably estimate its expectation. This holds for large $h$.

\begin{restatable}[Consistency of Projected Node Covariance]{proposition}{projectedconsistency}\label{prop:projected_consistency}
Let $\mX \in \mathbb{R}^{n \times d}$ be node features. For a projection dimension $h$, let $\mC \in \mathbb{R}^{d \times h}$ be such that $\text{vec}(\mC) \sim \mathcal{N}(\mathbf{0}, \mI_{d h})$. Define the stochastic node-covariance operator $\mK_h^{(0)} = \text{NodeCov}(\mX\mC) = \frac{1}{h}(\mPi_c \mX\mC)(\mPi_c \mX\mC)^T$, where $\mPi_c$ is the centering matrix.
Then, $\mK_h^{(0)}$ converges in probability to its expected value as  $h \to \infty$:
\begin{align}
\mK_h^{(0)} \xrightarrow{p} \mathbb{E}_{\mC}[\mK_h^{(0)}] = \mPi_c \mX \mX^T \mPi_c \quad \text{as } h \to \infty.
\end{align}
\end{restatable}

This consistency connects theory to practice. It shows that for a sufficiently large $h$, the operator accurately reflects the stable expected operator $\mPi_c \mX \mX^T \mPi_c$. Therefore, if two graphs have aligned expected operators (due to shared properties), using a large enough $h$ allows \ourmethod to effectively leverage these shared underlying structures, facilitating transfer across disparate feature spaces.

%\newpage

\begin{table*}[t]
    \centering
    \footnotesize
    \caption{{Performance of \ourmethod on pre-training datasets compared to \ourmethod-\textsc{specialized} which is trained separately on each individual dataset. \ourmethod maintains highly competitive performance.  }}
    \label{tab:common}
    \setlength{\tabcolsep}{3pt}
    \begin{adjustbox}{width=\textwidth}
    \begin{tabular}{lccccccccc}
    \toprule
    Method & \textsc{zinc} & \textsc{molesol} & \textsc{molhiv} & \textsc{moltox21} & \textsc{mnist} & \textsc{cifar10} & \textsc{ModelNet} & \textsc{Cuneiform} & \textsc{MSRC 21} \\
    & (\textsc{mae} $\downarrow$) & (\textsc{rmse} $\downarrow$) & (\textsc{roc-auc} $\uparrow$) & (\textsc{roc-auc} $\uparrow$) & (\textsc{acc} $\uparrow$) & (\textsc{acc} $\uparrow$) & (\textsc{acc} $\uparrow$) & (\textsc{acc} $\uparrow$) & (\textsc{acc} $\uparrow$) \\
    \midrule
    \textbf{\textsc{Trained per dataset}}\\
    $\,$ \ourmethod-\textsc{specialized} (0 props) & 0.1480 & 1.22 & 72.65 & 69.37 & 94.03 & 39.96 & 37.24 & 85.19 & 91.65 \\
    $\,$ \ourmethod-\textsc{specialized} & 0.1195 & 1.19 & 73.78 & 70.04 & 94.77 & 40.03 & 39.81 & 87.20 & 94.16 \\ 
    \midrule
    \textbf{\textsc{Trained on all datasets}}\\
    $\,$ 
    \ourmethod (0 props) & 0.1557 & 1.28 & 72.74 & 68.19 & 94.57 & 40.11 & 37.11 & 89.88 & 97.51 \\
    $\,$ \ourmethod & 0.1237 & 1.29 & 74.49 & 68.20 & 95.22 & 40.08 &   39.37 & 91.17 & 98.08 \\
    \bottomrule
    \end{tabular}
    \end{adjustbox}
\end{table*}

\begin{table*}[t]
    \centering
    \scriptsize
    
    \caption{Performance on unseen node classification datasets with new input features. \ourmethod effectively transfers to new datasets with new features, often outperforming or matching SOTA.}
    \label{tab:foundation}
    \begin{tabular}{lccc}
    \toprule
         Method & \textsc{Cora} & \textsc{CiteSeer} & \textsc{PubMed} 
         \\
         & (\textsc{acc} $\uparrow$) & (\textsc{acc} $\uparrow$) & (\textsc{acc} $\uparrow$)\\
         \midrule
          {\textbf{\textsc{Non-Parametric Baselines}}} \\
          $\,$ \textsc{Label propagation}~\citep{zhu2002learning} &  69.20 $\pm$\phantom{1}0.00 & 51.30 $\pm$\phantom{1}0.00& 71.40 $\pm$\phantom{1}0.00\\
         \midrule
         \textbf{\textsc{Supervised Baselines}} \\
         $\,$ \textsc{MLP} &  48.42 $\pm$\phantom{1}0.63 & 48.56 $\pm$\phantom{1}0.27 & 66.26 $\pm$\phantom{1}1.53   
         \\
         $\,$ \textsc{GCN} \citep{kipf2017semisupervised} & 78.86 $\pm$\phantom{1}1.48 &   64.52 $\pm$\phantom{1}0.89 & 74.49 $\pm$\phantom{1}0.99  \\
         $\,$ \textsc{GIN}~\citep{xu2019how} & 67.10 $\pm$\phantom{1}3.00 & 58.80 $\pm$\phantom{1}2.20 & 68.40 $\pm$\phantom{1}2.70\\
        \midrule 

         \textbf{\textsc{LLM-augmented GNNs}}\\
         $\,$  \textsc{OFA}~\citep{liu2024one} & 76.10 $\pm$\phantom{1}4.11 & {73.04 $\pm$\phantom{1}2.88} & {75.61 $\pm$\phantom{1}5.06}  
         
         \\
         $\,$ \textsc{GLEM-LM}~\citep{chen2024exploring} & 67.55 $\pm$\phantom{1}3.53 & 66.00 $\pm$\phantom{1}5.66 & 62.12 $\pm$\phantom{1}0.07
         
         \\
         
         \midrule
         \textbf{\textsc{LLM-based}}\\
         $\,$ \textsc{GraphText}~\citep{zhao2023graphtext} & 75.41 $\pm$\phantom{1}2.08 & 58.24 $\pm$\phantom{1}0.26 & {63.70 $\pm$\phantom{1}0.29} 
         \\
        $\,$ \textsc{RWNN-Llama3-8B}~\citep{kim2024revisiting} & 72.29 & N/A & {N/A} \\
        
        \midrule
        \textbf{\textsc{GNN-based}}\\
         
        $\,$ \textsc{AnyGraph}~\citep{xia2024anygraph} & {62.60 $\pm$\phantom{1}0.14} & {19.32 $\pm$\phantom{1}0.37} & {70.73 $\pm$\phantom{1}4.13} 
        \\
        $\,$ \textsc{GraphAny}~\citep{zhao2024graphany} & 79.36 $\pm$\phantom{1}0.23 & {68.42 $\pm$\phantom{1}0.39} & {76.30 $\pm$\phantom{1}0.41} 
        \\
        $\,$ \textsc{MDGPT}~\citep{yu2024textfree} & {43.36 $\pm$\phantom{1}8.92} & {42.50 $\pm$\phantom{1}9.78} & {51.91 $\pm$\phantom{1}9.00} 
        \\
        $\,$ \textsc{GCOPE}~\citep{zhao2024oneforall} & {35.54 $\pm$\phantom{1}2.09} & {31.18 $\pm$\phantom{1}4.35} & {32.87 $\pm$\phantom{1}4.08} 
        \\
        $\,$ \textsc{GPPT}~\citep{sun2022gppt} & 43.15 $\pm$\phantom{1}9.44 & 37.26 $\pm$\phantom{1}6.17 & 48.31 $\pm$17.72 \\
        $\,$
        \textsc{All-In-One}~\citep{sun2023all} & 52.39 $\pm$10.17 & 40.41 $\pm$\phantom{1}2.80 & 45.17 $\pm$\phantom{1}6.45\\
        $\,$
        \textsc{GPrompt}~\citep{gong2024self} & 56.66 $\pm$11.22 & 53.21 $\pm$10.94 & 39.74 $\pm$15.35 \\
        $\,$
        \textsc{GPF}~\citep{fang2023universal} & 38.57 $\pm$\phantom{1}5.41 &  31.16 $\pm$\phantom{1}8.05&  49.99 $\pm$\phantom{1}8.86 \\
        $\,$
        \textsc{GPF-plus}~\citep{fang2023universal} & 55.77 $\pm$10.30 & 59.67 $\pm$11.87 & 46.64 $\pm$18.97\\
        $\,$
        \textsc{ULTRA (3g)}~\citep{galkin2024towards} & 79.40 $\pm$\phantom{1}0.00 & 67.40 $\pm$\phantom{1}0.00 & 77.90 $\pm$\phantom{1}0.00\\
        $\,$ \textsc{SCORE}~\citep{wang2024towards} & 81.80 $\pm$\phantom{1}1.02 & 71.33 $\pm$\phantom{1}0.27 & 82.93 $\pm$\phantom{1}0.55\\
    $\,$  OpenGraph \citep{xia2024opengraph} & N/A & 58.58 & 58.40
        \\
        $\,$ 
        RiemannGFM~\citep{sun2025riemanngfm} & N/A & 66.38 & 76.20 \\ 
 $\,$  AutoGFM~\citep{chen2025autogfm} & 80.32 $\pm$1.12 & N/A & 78.28 $\pm$1.40
        \\
        \midrule
    
    $\,$ \ourmethod (0 props) & 79.26 $\pm$ 1.08 &  65.96 $\pm$ 1.25  & 77.30 $\pm$ 0.47  \\ 
    $\,$ \ourmethod  & 82.13 $\pm$ 0.97 & 69.12 $\pm$ 0.89 & 78.03 $\pm$ 0.82    \\ 
    \bottomrule
         
    \end{tabular}
\end{table*}

\section{Experiments}
\label{sec:experiments}
In this section, we empirically evaluate the ability of \ourmethod to learn transferable representations from diverse graph datasets, and, critically, its capability to generalize to new datasets presenting entirely new input features. Our experiments are designed to answer two primary research questions:
\begin{enumerate}[label=(\textbf{Q\arabic*})]
    \item How does a single \ourmethod model, pre-trained jointly on a diverse collection of graph datasets (each with its own input features and task), perform on these individual source datasets compared to {training a separate model for each dataset}?
    \item How effectively do the representations learned by a pre-trained \ourmethod model transfer to new, unseen datasets that may have entirely different input features and downstream tasks?
\end{enumerate}
Next, we report our main experiments and refer to \Cref{app:results} for additional results (including time complexity). Implementation details, dataset statistics, and hyperparameter configurations are in \Cref{app:implementation,sec:app:data}.

\subsection{Performance on Pre-training Source Datasets (A1)}\label{subsection:pretraining}
In this subsection, we assess the ability of \ourmethod to learn from a wide array of source datasets simultaneously, without significant performance degradation on the individual datasets it was pre-trained on. This is needed for establishing its viability to obtain general-purpose pre-trained representations.

{To test this, w}e pre-train a single \ourmethod encoder on a diverse corpus of nine graph datasets, encompassing molecular data (\textsc{zinc}~\citep{dwivedi2023benchmarking}, \textsc{ogbg-molhiv}~\citep{hu2020open}, \textsc{ogbg-molesol}~\citep{hu2020open}, \textsc{ogbg-moltox21}~\citep{hu2020open}), computer vision derived graphs (\textsc{mnist}~\citep{dwivedi2023benchmarking}, \textsc{cifar10}~\citep{dwivedi2023benchmarking}, \textsc{Cuneiform}~\citep{MorrisTUD}, \textsc{MSRC 21}~\citep{MorrisTUD}), and 3d shape (\textsc{ModelNet}~\citep{wu20153d}) with varying tasks (classification and regression) and heterogeneous input features (differing dimensionalities, types, value ranges and semantics). For each dataset-task pair, a dedicated prediction head is attached to the shared \ourmethod component and trained to predict the corresponding target.
{We compare this single, jointly-trained model against its specialist counterparts: nine separate instances of the \ourmethod architecture, each trained from scratch on only one of the source datasets (\ourmethod-\textsc{specialized}).}

{\textbf{Results and Discussion.} 
\Cref{tab:common} confirms that \ourmethod not only successfully operates across datasets with heterogeneous features but is also highly effective, achieving performance competitive with, and at times superior to, specialized models. While the individually trained \ourmethod-\textsc{specialized} holds a slight edge on \textsc{zinc}, 
\textsc{molesol}, 
\textsc{moltox21}, 
and \textsc{ModelNet}, 
the jointly-trained \ourmethod demonstrates superior performance on the remaining 5 datasets. This advantage is particularly notable on \textsc{Cuneiform} (91.17\% vs. 87.20\%) and \textsc{MSRC 21} (98.08\% vs. 94.16\%), while also outperforming \ourmethod-\textsc{specialized} on \textsc{molhiv}, 
\textsc{mnist}, 
and \textsc{cifar10}. 
We also observe a clear advantage to using propagated operators, as the full model generally outperforms the (0 props) variant (a version computed without propagated covariance operators) across the tasks.

Overall, these results strongly indicate that a single, jointly pre-trained \ourmethod encoder can learn general-purpose representations from diverse data that remain highly competitive with, and in several cases surpass, those obtained when learning on a single task.
}

\subsection{Transferability to Unseen Datasets and Input Features (A2)}
\label{sec:exp:transfer}

This subsection assesses the central hypothesis underlying our research: namely, that a single, pre-trained \ourmethod model can effectively generalize to novel datasets characterized by distinct input features. To evaluate this hypothesis, we maintain the pre-trained \ourmethod encoder frozen, thereby ensuring that its learned representations remain unchanged. For each new dataset, which encompasses a range of node and graph-level tasks and introduces previously unseen input features and target label schemas, we instantiate and train a new prediction head using the frozen representations extracted by \ourmethod. This approach enables us to isolate the generalizability of \ourmethod's pre-trained representations, providing a test of its ability to adapt to unfamiliar data distributions.

\begin{wraptable}[22]{r}{0.55\linewidth} 
\vspace{-10pt}
    \centering
    \scriptsize
    \caption{Performance on unseen graph classification datasets with new input features. \ourmethod demonstrates strong transferability to graph-level tasks with new features, underscoring its versatility across different tasks and its ability to handle different features.}
    \begin{tabular}{lccc}
    \toprule
       Dataset  &   \textsc{MUTAG}  & \textsc{PROTEINS} \\
       & (\textsc{acc} $\uparrow$) & (\textsc{acc} $\uparrow$) \\
       \midrule
       \textbf{\textsc{Supervised Baselines}} & \\
       $\,$ MLP & 67.20 $\pm$\phantom{1}1.00 & 59.20 $\pm$\phantom{1}1.00\\
       $\,$ GIN~\citep{xu2019how} & 89.40 $\pm$\phantom{1}5.60 & 76.20 $\pm$\phantom{1}2.80\\
       \midrule
       \textbf{\textsc{LLM-augmented GNNs}} &\\
       $\,$ \textsc{OFA}~\citep{liu2024one}  & 61.04 $\pm$\phantom{1}4.71 & 61.40 $\pm$\phantom{1}2.99\\
       \midrule
       
       \textbf{\textsc{GNN-based}} & \\
       $\,$
       \textsc{MDGPT}~\citep{yu2024textfree} & 57.36 $\pm$14.26 & 54.35 $\pm$10.26\\
       $\,$ \textsc{GPPT~\citep{sun2022gppt}} & 60.40 $\pm$15.43  & 60.92 $\pm$\phantom{1}2.47 \\
       $\,$ \textsc{All-In-One}~\citep{sun2023all} & 79.87 $\pm$\phantom{1}5.34 & 66.49 $\pm$\phantom{1}6.26 \\
       $\,$ \textsc{GPrompt}~\citep{gong2024self} & 73.60 $\pm$\phantom{1}4.76 &  59.17 $\pm$11.26\\
       $\,$ \textsc{GPF}~\citep{fang2023universal} & 68.40 $\pm$\phantom{1}5.09 & 63.91 $\pm$\phantom{1}3.26 \\
       $\,$ \textsc{GPF-plus}~\citep{fang2023universal} & 65.20 $\pm$\phantom{1}6.94 &  62.92 $\pm$\phantom{1}2.78\\
       $\,$ \textsc{ULTRA(3g)}~\citep{galkin2024towards} & 63.33 $\pm$\phantom{1}0.00 & 58.09 $\pm$\phantom{1}0.00\\
       $\,$ \textsc{SCORE}~\citep{wang2024towards} & 85.33 $\pm$\phantom{1}2.11 & 68.54 $\pm$\phantom{1}1.47\\
       \midrule 
       $\,$ \ourmethod (0 props) &  92.50 $\pm$\phantom{1}6.60 &  76.72 $\pm$\phantom{1}3.19 \\
       $\,$ \ourmethod & 92.90 $\pm$\phantom{1}6.34 & 78.20 $\pm$\phantom{1}3.81 \\
       \bottomrule
    \end{tabular}
    \label{tab:graph_class}
\end{wraptable}

We compare \ourmethod against several categories of baselines: {(1) the non-parameteric baseline Label Propagation~\citep{zhu2002learning}, on node classification tasks where it is applicable;} (2) standard supervised GNNs trained from scratch on the target datasets; (3) LLM-augmented GNNs; (4) LLM-based methods; and (5) other GNN-based foundation models or transfer learning approaches. We adhere to their prescribed protocols for adaptation on new datasets. We refer the reader to \Cref{app:results} for this categorization.

\textbf{Results and Discussions.} \ourmethod demonstrates robust transferability across both node-level (\Cref{tab:foundation}) and graph-level (\Cref{tab:graph_class}) tasks on unseen datasets with new input features. \ourmethod not only significantly surpasses the performance of standard supervised GNNs trained from scratch on these target datasets, but also outperforms recent state-of-the-art graph foundation models.
On node classification benchmarks (\Cref{tab:foundation}),
\ourmethod consistently demonstrates strong transfer capabilities. 
For instance, on \textsc{Cora}~\citep{kipf2017semisupervised}, it obtains an accuracy of 82.13\% which not only surpasses standard supervised \textsc{GCN} (78.86\%), but it also exceeds leading baselines like \textsc{SCORE}~\citep{wang2024towards} (81.80\%) and \textsc{GraphAny}~\citep{zhao2024graphany} (79.36\%). 
This strong performance extends to graph classification tasks (\Cref{tab:graph_class}). On MUTAG~\citep{MorrisTUD},  \ourmethod achieves 92.90\% accuracy, exceeding both the supervised GIN baseline (89.40\%) and state of the art methods like \textsc{SCORE} (85.33\%) and \textsc{All-In-One}~\citep{sun2023all} (79.87\%). Furthermore, consistent with observations on the source datasets in \Cref{subsection:pretraining}, the inclusion of propagated covariance operators in \ourmethod  enhances transfer performance compared to \ourmethod (0 props).

These results provide evidence that a single pre-trained \ourmethod encoder produces effective, general-purpose representations. These representations readily adapt to both node and graph-level tasks on new datasets with new features, maintaining a versatility that provides a strong advantage over specialized models (\textsc{GraphAny}~\citep{zhao2024graphany}, \textsc{GraphText}~\citep{zhao2023graphtext}, \textsc{GCOPE}~\citep{zhao2024oneforall}, \textsc{AnyGraph}~\citep{xia2024anygraph}), only supporting node classification.

\section{Conclusion}
\label{sec:conclusion}

Input feature heterogeneity critically limits the development of Graph Foundation Models (GFMs). Our \ourmethod offers a theoretically-grounded solution, processing arbitrary node features through stochastic projections and node-covariance operators to build robust representations independent of the original feature space. We prove that these representations achieve distributional invariance to input feature permutations, and their underlying expected operator is invariant to orthogonal basis changes, thereby helping capture robust intrinsic structures of the data. The empirical transfer performance of \ourmethod across new datasets with disparate features demonstrates its potential to mitigate the challenges posed by feature heterogeneity, contributing to the development of GFMs.

\clearpage

\textbf{Limitations and Future Work.} 
The scalability of \ourmethod on extremely large graphs may be constrained by its dense covariance operators, in case direct access to the covariance operators are required, similarly to graph transformers; developing sparse approximations presents a key avenue for future research. Another promising direction involves exploring structured or learnable input feature projections as alternatives to the random Gaussian projections. Notably, as discussed in \Cref{app:complexity}, in common GNNs, we can avoid the storage of dense covariance operators, thereby achieving improved scalability.

{
\paragraph{Reproducibility Statement.}
Our code is available at \url{https://github.com/MosheEliasof/ALLIN}. We carefully document dataset details in \Cref{sec:app:data} and implementation details in \Cref{app:implementation}.

\paragraph{Ethics Statement.}
Our work is primarily methodological and presents minimal direct ethical concerns. All experiments are conducted on publicly available benchmark datasets widely used in the graph machine learning community, and we have used these datasets in accordance with their established licensing and terms of use. While our contribution is foundational, we advocate for the responsible application of transferable graph models. We caution against their use in analyzing sensitive social or personal data without appropriate safeguards and ethical oversight.

\paragraph{Usage of Large Language Models in This Work.}
LLMs were used in this work for text editing suggestions. All concepts, theoretical analysis, code development, and original writing were carried out by the authors.

%\paragraph{Acknowledgments.}
%\textcolor{red}{Add acknowledgement here}
%}

% \clearpage
% \newpage

\bibliographystyle{iclr2026_conference}
\bibliography{references}

% \clearpage
\newpage

\appendix

\section{Additional Related Work}

\paragraph{Generalization Theory of MPNNs.} Significant theoretical progress has advanced our understanding of generalization in Message Passing Neural Networks (MPNNs). As discussed in recent surveys~\citep{vasileiou2025survey,zhang2025surveydeepgraphlearning}, these efforts often focus on how architectures and graph properties (such as maximum degree) influence the generalization gap, employing analytical tools like Rademacher complexity and PAC-Bayesian analysis to derive performance bounds \citep{garg2020generalization, liao2021pac}. Other lines of work, leveraging concepts like covering numbers or graphon theory, investigate model stability and generalization under shifts in graph structure or topology, particularly in large-scale or evolving graph scenarios~\citep{levie2023graphon, vasileiou2024covered}. While these foundational theories provide important insights into GNN expressivity and their ability to generalize, especially concerning structural variations, they typically assume a consistent definition of the input feature space across different graphs. The cross-dataset generalization challenge that \ourmethod addresses is distinct: we specifically tackle scenarios where graphs present node features from entirely different feature spaces, potentially varying in both the number of available features (dimensionality) and their semantic meaning between train (source) and test (target) graphs. Our theoretical framework (\Cref{sec:theory}) therefore focuses on establishing principles for robustness and transferability under such input feature space heterogeneity, aiming to complement existing generalization theories that predominantly address structural changes.

\paragraph{Additional Efforts towards Graph Foundation Models.} Another significant challenge in graph transfer learning arises in settings like heterogeneous knowledge graphs, where models must generalize to unseen entities and relation types. Approaches such as ISDEA+~\citep{gao2023double} and MTDEA~\citep{zhou2023multi} tackle this by employing set aggregation techniques over representations specific to edge types, aiming for equivariance to permutations of these types, supported by a ``double equivariance'' theoretical framework. Similarly, methods like InGram~\citep{lee2023ingram}, ULTRA~\citep{galkin2024towards}, TRIX~\citep{zhang2024trix}, and MOTIF~\citep{huang2025expressive} construct explicit ``relation graphs'' to model interactions among different edge types. These works provide valuable solutions for structural and relational heterogeneity. In contrast, \ourmethod primarily addresses the distinct challenge of heterogeneity in input features, that is, varying feature dimensionalities and semantics across graphs. While the aforementioned methods focus on generalizing over graph schema and relation types (often assuming node features are not present), \ourmethod directly processes arbitrary node features to derive transferable node-covariance operators and representations. Other efforts in graph representation learning aim for transferability across diverse graph tasks. For example, HoloGNN~\citep{bevilacqua2025holographic} proposes a framework to learn node representations that can be applied to various downstream tasks on a given graph or graphs. However, such approaches typically assume that the underlying node feature space remains consistent across these tasks. \ourmethod, conversely, is specifically designed to address the challenge of generalizing to new and unseen datasets where the node features themselves can differ fundamentally in dimensionality and semantics, a problem distinct from task-level transfer within a fixed feature domain.

\section{Additional Theoretical Considerations and Proofs}\label{app:theory}

\distributional*
\begin{proof}
Let $\mC$ have columns $\vc_1, \dots, \vc_h$. Since the entries $C_{ik}$ are i.i.d $\mathcal{N}(0, 1)$, each column $\vc_j \sim \mathcal{N}(\mathbf{0}, \mI_d)$ and the columns are mutually independent.

Consider the matrix $\bar{\mC} = \mP^T \mC$. Since $\mP$ is a permutation matrix, $\mP^T$ is also a permutation matrix and is orthogonal, that is $\mP^T (\mP^T)^T = \mP^T \mP = \mI_d$.

The columns of $\bar{\mC}$ are $\bar{\vc}_j = \mP^T \vc_j$. Since $\vc_j \sim \mathcal{N}(\mathbf{0}, \mI_d)$ and $\mP^T$ is orthogonal, then
\begin{align}
    \bar{\vc}_j \sim \mathcal{N}(\mP^T\mathbf{0}, \mP^T \mI_d (\mP^T)^T) = \mathcal{N}(\mathbf{0}, \mP^T \mP) = \mathcal{N}(\mathbf{0}, \mI_d)
\end{align}

Furthermore, since $\vc_1, \dots, \vc_h$ are independent, the transformed columns $\bar{\vc}_1, \dots, \bar{\vc}_h$ are also independent. Thus, the matrix $\bar{\mC}$ has the same distribution as $\mC$, i.e., $\bar{\mC} \stackrel{d}{=} \mC$.

Now consider $\bar{\mR}^{(0)} = (\mX\mP)\mC$. Since $\mC \stackrel{d}{=} \bar{\mC}$, we can write:
$$\bar{\mR}^{(0)} \stackrel{d}{=} (\mX\mP)\bar{\mC}$$
Substitute $\bar{\mC} = \mP^T \mC$:
$$\bar{\mR}^{(0)} \stackrel{d}{=} (\mX\mP)(\mP^T \mC) = \mX(\mP\mP^T)\mC$$
Since $\mP$ is orthogonal, $\mP\mP^T = \mI_d$.
$$\bar{\mR}^{(0)} \stackrel{d}{=} \mX \mI_d \mC = \mX\mC = \mR$$
Thus, $\mR$ and $\bar{\mR}^{(0)}$ are equal in distribution.
\end{proof}

\distributionalnodecov*
\begin{proof}
Let $g_p(\mR^{(0)}) = \text{NodeCov}(\mA^p \mR^{(0)})$ be the deterministic function that computes the p-th order node covariance operator from the initial projected features $\mR^{(0)}$. From \Cref{prop:distributional_invariance}, we have $\mR^{(0)} \stackrel{d}{=} \bar{\mR}^{(0)}$.
Since applying a deterministic function $g_p$ to random variables that are equal in distribution results in outputs that are equal in distribution, we have $g_p(\mR^{(0)}) \stackrel{d}{=} g_p(\bar{\mR}^{(0)})$, which means $\mK^{(p)} \stackrel{d}{=} \bar{\mK}^{(p)}$ for each $p = 0 \ldots k$.
Furthermore, since all operators $\mK^{(p)}$ in $\mathcal{K}$ are derived from the same $\mR^{(0)}$, and all operators $\bar{\mK}^{(p)}$ in $\mathcal{\bar{K}}$ are derived from $\bar{\mR}^{(0)}$, the distributional equality extends to the joint distribution of the sets:
$\mathcal{K} \stackrel{d}{=} \mathcal{\bar{K}}$.
\end{proof}

\begin{restatable}[Distributional Invariance of Hidden Representations to Input Permutation]{theorem}{hiddeninvariance}\label{thm:hidden_invariance}
Let $\mX \in \mathbb{R}^{n \times d}$ be node features, and $\mP \in \mathbb{R}^{d \times d}$ be any permutation matrix. Let $\mR^{(0)} = \mX\mC$ be the initial projected features, and $\mathcal{K} = \{\mK^{(p)}\}_{p=0}^k$ be the set of node-covariance operators. 
Let the initial hidden representation be $\mH^{(0)} = \mR^{(0)} \oplus \mS$, where $\mS$ is a structural encoding matrix independent of $\mX$.
Subsequent hidden representations $\mH^{(\ell)}$ for $\ell = 1, \ldots, L$ are computed by a deterministic GNN layer function.

The initial hidden representation $\mH^{(0)}$ and all subsequent hidden representations $\mH^{(\ell)}$ for $\ell = 1, \ldots, L$ are invariant in distribution to permutations of the input features $\mX$. That is, if $\bar{\mH}^{(\ell)}$ are the representations computed using $\mX\mP$ instead of $\mX$, then $\mH^{(\ell)} \stackrel{d}{=} \bar{\mH}^{(\ell)}$ for all $\ell$.
\end{restatable}
\begin{proof}
We proceed by induction on the layer index $\ell$.

\textbf{Base Case ($\ell = 0$).} Let $\mR^{(0)} = \mX\mC$ and $\bar{\mR}^{(0)} = (\mX\mP)\mC$. The initial hidden representations are $\mH^{(0)} = \mR^{(0)} \oplus \mS$ and $\bar{\mH}^{(0)} = \bar{\mR}^{(0)} \oplus \mS$.
From \Cref{prop:distributional_invariance}, we know that $\mR^{(0)} \stackrel{d}{=} \bar{\mR}^{(0)}$.
Since the structural encoding $\mS$ is assumed independent of $\mX$ (and thus fixed with respect to the permutation $\mP$), and the concatenation operation $\oplus$ is a deterministic function, applying this function preserves the distributional equality.
Therefore, $\mH^{(0)} = \mR^{(0)} \oplus \mS \stackrel{d}{=} \bar{\mR}^{(0)} \oplus \mS = \bar{\mH}^{(0)}$. The base case holds.

\textbf{Inductive Hypothesis.} Assume that for some layer $\ell - 1 \geq 0$, the hidden representations are equal in distribution: $\mH^{(\ell-1)} \stackrel{d}{=} \bar{\mH}^{(\ell-1)}$.

\textbf{Inductive Step (Layer $\ell$).} The hidden representations at layer $\ell$ are computed as:
$$\mH^{(\ell)} = F_\ell(\mH^{(\ell-1)}, \mathcal{O})$$
$$\bar{\mH}^{(\ell)} = F_\ell(\bar{\mH}^{(\ell-1)}, \bar{\mathcal{O}})$$
where $F_\ell$ represents the deterministic computation performed by the $\ell$-th GNN layer (given fixed learned weights), $\mathcal{O} = \{\mI, \mA\} \cup \mathcal{K}$ with $\mathcal{K} = \{\text{NodeCov}(\mA^p \mR^{(0)})\}_{p=0}^k$, and $\bar{\mathcal{O}} = \{\mI, \mA\} \cup \bar{\mathcal{K}}$ with $\bar{\mathcal{K}} = \{\text{NodeCov}(\mA^p \bar{\mR}^{(0)})\}_{p=0}^k$.

From \Cref{cor:distributional_invariance_K}, we know that the set of random operators $\mathcal{K}$ is equal in distribution to  $\mathcal{\bar{K}}$, i.e., $\mathcal{K} \stackrel{d}{=} \mathcal{\bar{K}}$. Since $\mI$ and $\mA$ are fixed, the full set of operators used by the layer also satisfies $\mathcal{O} \stackrel{d}{=} \mathcal{\bar{O}}$.

Now consider the inputs to the function $F_\ell$. The pair $(\mH^{(\ell-1)}, \mathcal{O})$ determines $\mH^{(\ell)}$, and the pair $(\bar{\mH}^{(\ell-1)}, \bar{\mathcal{O}})$ determines $\bar{\mH}^{(\ell)}$.
Both $\mH^{(\ell-1)}$ and $\mathcal{O}$ are deterministic functions of the initial projection $\mR^{(0)}$ (and fixed elements $\mS, \mA, \mI$, and layer weights). Let $J$ be the function representing the computation up to layer $\ell - 1$ and the computation of operators, such that $(\mH^{(\ell-1)}, \mathcal{O}) = J(\mR^{(0)}, \mS, \mA, \mI, \text{Weights})$
Similarly, $(\bar{\mH}^{(\ell-1)}, \bar{\mathcal{O}}) = J(\bar{\mR}^{(0)}, \mS, \mA, \mI, \text{Weights})$.

Since $\mR^{(0)} \stackrel{d}{=} \bar{\mR}^{(0)}$ (\Cref{prop:distributional_invariance}) and $J$ is a deterministic function, it follows that the joint distribution of the outputs is preserved:
$$(\mH^{(\ell-1)}, \mathcal{O}) \stackrel{d}{=} (\bar{\mH}^{(\ell-1)}, \bar{\mathcal{O}})$$
This establishes that the inputs to the deterministic layer function $F_\ell$ are equal in distribution. Applying the deterministic function $F_\ell$ preserves this equality:
$$\mH^{(\ell)} = F_\ell(\mH^{(\ell-1)}, \mathcal{O}) \stackrel{d}{=} F_\ell(\bar{\mH}^{(\ell-1)}, \bar{\mathcal{O}}) = \bar{\mH}^{(\ell)}$$
Thus, the inductive step holds.
\end{proof}

\expressivity*
\begin{proof}
We will show that there exists $\mX$, $u$, $v$ such that
\begin{enumerate*}[label=(\arabic*)]
    \item  nodes $u$ and $v$ are automorphic within $\text{NodeCov}(\mX)$, and consequently, the GNN, when using $\text{NodeCov}(\mX)$ as the operator and identical initial embeddings, produces identical final representations for these nodes.
    \item For the same $\mX$, with probability 1 (over the draw of $\mC$), nodes $u$ and $v$ are \textbf{not} automorphic and therefore distinguishable in $\text{NodeCov}(\mX\mC)$. 
\end{enumerate*}
We provide a constructive example. Let $n=3$ nodes $\{u,v,w\}$ and $d=3$ features. Consider the feature matrix $\mX$:
$$\mX = \begin{pmatrix} \mX_u^T \\ \mX_v^T \\ \mX_w^T \end{pmatrix} = \begin{pmatrix} 1 & 0 & 1 \\ 0 & 1 & 1 \\ 1 & 1 & 0 \end{pmatrix}$$
Here, $\mX_u = (1,0,1)^T$, $\mX_v = (0,1,1)^T$, and $\mX_w = (1,1,0)^T$. Clearly, $\mX_u \neq \mX_v$.

\textbf{Proof for item (1).}
The column means of $\mX$ are $\bar{\mX}_{\text{col}} = (2/3, 2/3, 2/3)^T$.
The centered feature matrix $\mX_c = \mPi_c \mX = \mX - \mathbf{1}_3 \bar{\mX}_{\text{col}}^T$ is:
$$\mX_c = \begin{pmatrix} 1/3 & -2/3 & 1/3 \\ -2/3 & 1/3 & 1/3 \\ 1/3 & 1/3 & -2/3 \end{pmatrix}$$
Then $$\text{NodeCov}(\mX) = \begin{pmatrix} 2/9 & -1/9 & -1/9 \\ -1/9 & 2/9 & -1/9 \\ -1/9 & -1/9 & 2/9 \end{pmatrix}.$$
In the weighted graph defined by $\text{NodeCov}(\mX)$, all nodes are automorphic to each other. 
If a GNN uses $\text{NodeCov}(\mX)$ as its feature-derived operator and starts with identical initial embeddings for all nodes, standard message passing layers will preserve this symmetry, leading to identical final representations $\mH_u^{(L)} = \mH_v^{(L)} = \mH_w^{(L)}$. Thus, such a GNN cannot distinguish $u$ from $v$.

\textbf{Proof for item (2). } Let $\mR^{(0)} = \mX\mC$. The rows of $\mR^{(0)}$ are $\mR^{(0)}_u = \mX_u^T\mC$, $\mR^{(0)}_v = \mX_v^T\mC$, $\mR^{(0)}_w = \mX_w^T\mC$. Since $\mX_u \neq \mX_v$ and $\mC$ is drawn from a continuous distribution (Gaussian entries), $\mX_u^T\mC \neq \mX_v^T\mC$ with probability 1. Thus, $\mR^{(0)}_u \neq \mR^{(0)}_v$ almost surely.
Let $\mR^{(0)}_c = \mPi_c \mR^{(0)}$. The rows of $\mR_c$ are $\mR^{(0)}_{c,u}, \mR^{(0)}_{c,v}, \mR^{(0)}_{c,w}$. Since $\mR^{(0)}_u \neq \mR^{(0)}_v$, it follows that $\mR^{(0)}_{c,u} \neq \mR^{(0)}_{c,v}$ almost surely (unless $\mPi_c$ projects their difference to zero, which is a measure zero event for a fixed $\mX$ and random $\mC$).
The operator is $\mK^{(0)} =
 \text{NodeConv}(\mX\mC) = \frac{1}{h}\mR_c \mR_c^T$.
An element $(\mK^{(0)})_{ij} = \frac{1}{h} \mR^{(0)}_{c,i} \cdot \mR^{(0)}_{c,j}$.
Consider the specific symmetry that existed for $\text{NodeConv}(\mX)$, e.g., $(\text{NodeConv}(\mX))_{uw} = (\text{NodeConv}(\mX))_{vw} = -1/9$.
For $\mK^{(0)}$, we compare $(\mK^{(0)})_{uw} = \frac{1}{h} \mR^{(0)}_{c,u} \cdot \mR^{(0)}_{c,w}$ and $(\mK^{(0)})_{vw} = \frac{1}{h} \mR^{(0)}_{c,v} \cdot \mR^{(0)}_{c,w}$.
These are equal if $(\mR^{(0)}_{c,u} - \mR^{(0)}_{c,v}) \cdot \mR^{(0)}_{c,w} = 0$.
Since $\mR^{(0)}_{c,u} - \mR^{(0)}_{c,v} \neq \mathbf{0}$ almost surely, and $\mR^{(0)}_{c,w}$ is a random vector (whose distribution depends on $\mC$), the event that their dot product is exactly zero has probability 0 for continuous distributions unless one of them is deterministically zero (which is not the case here a.s.).
Therefore, with probability 1, $(\mK^{(0)})_{uw} \neq (\mK^{(0)})_{vw}$. This breaks the specific symmetry that made node u and node v have equivalent relational profiles to node w in $\text{NodeCov}(\mX)$.
More generally, the matrix $\mK^{(0)}$ will not, with probability 1, exhibit the high degree of symmetry found in $\text{NodeCov}(\mX)$ for this specific $\mX$. Thus, nodes $u$ and $v$ will generally not be automorphic with respect to $\mK^{(0)}$ in the same way they were for $\text{NodeCov}(\mX)$.
A GNN using this specific realization $\mK^{(0)}$ (and identical initial embeddings, can now potentially produce $\mH_u^{(L)} \neq \mH_v^{(L)}$ because the operator $\mK^{(0)}$ provides different relational information for $u$ and $v$. 

\end{proof}

\orthogonaltransformation*
\begin{proof}
Let $\mR^{(0)}  = \mX \mC$. Using the definition of the NodeCov operator and properties of the centering matrix $\mPi_c$:
\begin{align*}
\text{NodeCov}(\mR^{(0)} ) &= \frac{1}{h} (\mPi_c \mR^{(0)}) (\mPi_c \mR^{(0)} )^T \\
&= \frac{1}{h} \mPi_c (\mX \mC) (\mX \mC)^T \mPi_c^T \\
&= \frac{1}{h} \mPi_c \mX \mC \mC^T \mX^T \mPi_c
\end{align*}
Taking the expectation over $\mC$:
\begin{align*}
\E_{\mC}[\text{NodeCov}(\mX \mC)] &= \E_{\mC}\left[ \frac{1}{h} \mPi_c \mX \mC \mC^T \mX^T \mPi_c \right] \\
&= \frac{1}{h} \mPi_c \mX \E_{\mC}[\mC \mC^T] \mX^T \mPi_c \quad \text{(by linearity of expectation)}
\end{align*}
We evaluate $\E_{\mC}[\mC \mC^T]$. Let $\vc_j \in \mathbb{R}^d$ be the j-th column of $\mC$. Since the entries of $\mC$ are i.i.d. $\mathcal{N}(0,1)$, each column vector $\vc_j$ follows $\vc_j \sim \mathcal{N}(\mathbf{0}, \mI_d)$. Therefore, $E[\vc_j \vc_j^T] = \mI_d$. Using linearity of expectation:
\begin{align*}
\E_{\mC}[\mC \mC^T] = \E_{\mC}\left[\sum_{j=1}^h \vc_j \vc_j^T\right] = \sum_{j=1}^h \E_{\mC}[\vc_j \vc_j^T] = \sum_{j=1}^h \mI_d = h \mI_d
\end{align*}
Substituting this back:
\begin{align*}
\E_{\mC}[\text{NodeCov}(\mX \mC)] = \frac{1}{h} \mPi_c \mX (h \mI_d) \mX^T \mPi_c = \mPi_c \mX \mX^T \mPi_c
\end{align*}

Now consider the transformed features $\bar{\mX} = \mX \mQ$. Let $\bar{\mR}^{(0)}  = \bar{\mX} \mC = \mX \mQ \mC$. We compute $\E_{\mC}[\text{NodeCov}(\bar{\mR}^{(0)} )]$:
\begin{align*}
\text{NodeCov}(\bar{\mR}^{(0)}) &= \frac{1}{h} (\mPi_c \bar{\mR}^{(0)}) (\mPi_c \bar{\mR}^{(0)})^T \\
&= \frac{1}{h} \mPi_c (\mX \mQ \mC) (\mX \mQ \mC)^T \mPi_c \\
&= \frac{1}{h} \mPi_c \mX \mQ \mC \mC^T \mQ^T \mX^T \mPi_c
\end{align*}
Taking the expectation over $\mC$:
\begin{align*}
\E_{\mC}[\text{NodeCov}(\mX \mQ \mC)] &= \frac{1}{h} \mPi_c \mX \mQ \E_{\mC}[\mC \mC^T] \mQ^T \mX^T \mPi_c \\
&= \frac{1}{h} \mPi_c \mX \mQ (h \mI_d) \mQ^T \mX^T \mPi_c \quad \text{(using } E[\mC\mC^T] = h\mI_d \text{)} \\
&= \mPi_c \mX \mQ \mI_d \mQ^T \mX^T \mPi_c \\
&= \mPi_c \mX (\mQ \mQ^T) \mX^T \mPi_c \\
&= \mPi_c \mX \mI_d \mX^T \mPi_c \quad \text{(since } \mQ \text{ is orthogonal, } \mQ\mQ^T = \mI_d \text{)} \\
&= \mPi_c \mX \mX^T \mPi_c
\end{align*}
Thus, $\E_{\mC}[\text{NodeCov}(\mX \mQ \mC)] = \E_{\mC}[\text{NodeCov}(\mX\mC)] = \mPi_c \mX \mX^T \mPi_c$.
\end{proof}

\lossjensen*
\begin{proof}
The proof follows directly from Jensen's inequality for vector- or matrix-valued random variables.

Let the random variable be the final hidden representation $Z = \mH^{(L)}$, which is a function of the random projection matrix $\mC$.

By assumption, the function $f$ is convex with respect to its input argument $\mH^{(L)}$.
Jensen's inequality states that for a convex function $f$ and a random variable $Z$ with finite expectation, $f(\E[Z])\leq \E[f(Z)]$.
Applying this with $Z = \mH^{(L)}$ and the defined function $f$, we get:
\begin{align*}
    \mathcal{L}(\phi(\E_{\mC}[\mH^{(L)}]), \mY) \leq \E_{\mC}[\mathcal{L}(\phi(\mH^{(L)}), \mY)]
\end{align*}
which is the desired result.
\end{proof}

\projectedconsistency*
\begin{proof}
Let $\mC = [\vc_1, \dots, \vc_h]$ denote the random projection matrix, where each column $\vc_j \in \mathbb{R}^d$ is a random vector. Since the entries of $\mC$ are sampled i.i.d. from $\mathcal{N}(0, 1)$, the columns $\vc_j$ are independent and identically distributed according to $\vc_j \sim \mathcal{N}(\mathbf{0}, \mI_d)$.

The stochastic node-covariance operator $\mK_h^{(0)}$ (\Cref{eq:cov}) can be rewritten as:
\begin{align*}
\mK_h^{(0)} &= \frac{1}{h}(\mPi_c \mX\mC)(\mPi_c \mX\mC)^T \\
&= \frac{1}{h} (\mPi_c \mX [\vc_1, \dots, \vc_h]) (\mPi_c \mX [\vc_1, \dots, \vc_h])^T \\
&= \frac{1}{h} ([\mPi_c \mX \vc_1, \dots, \mPi_c \mX \vc_h]) ([\mPi_c \mX \vc_1, \dots, \mPi_c \mX \vc_h])^T \\
&= \frac{1}{h} \sum_{j=1}^h (\mPi_c \mX \vc_j) (\mPi_c \mX \vc_j)^T \quad \text{(using block matrix multiplication definition)}
\end{align*}

Let us define the random matrix $\mY_j \in \mathbb{R}^{n \times n}$ as:
$$\mY_j = (\mPi_c \mX \vc_j)(\mPi_c \mX \vc_j)^T$$
Since the columns $\vc_j$ are i.i.d. and $\mY_j$ is a fixed function of $\vc_j$ (given the fixed matrices $\mX$ and $\mPi_c$), the random matrices $\mY_1, \mY_2, \dots, \mY_h$ are also independent and identically distributed (i.i.d.).

The operator $\mK_h^{(0)}$ can thus be written as the sample mean of these i.i.d. random matrices:
$$\mK_h^{(0)} = \frac{1}{h} \sum_{j=1}^h \mY_j$$
Now, we compute the expected value of $\mY_j$. Using the linearity of expectation and the property that $\mPi_c$ and $\mX$ are constant with respect to the expectation over $\mC$ (and $\mPi_c = \mPi_c^T$):
\begin{align*}
\mathbb{E}[\mY_j] &= \mathbb{E}[(\mPi_c \mX \vc_j)(\mPi_c \mX \vc_j)^T] \\
&= \mathbb{E}[\mPi_c \mX \vc_j \vc_j^T \mX^T \mPi_c^T] \\
&= \mPi_c \mX \mathbb{E}[\vc_j \vc_j^T] \mX^T \mPi_c
\end{align*}
Since $\vc_j \sim \mathcal{N}(\mathbf{0}, \mI_d)$, we know that $\mathbb{E}[\vc_j \vc_j^T] = \text{Cov}(\vc_j) + \mathbb{E}[\vc_j]\mathbb{E}[\vc_j]^T = \mI_d + \mathbf{0}\mathbf{0}^T = \mI_d$. Substituting this in:
\begin{align*}
    \mathbb{E}[\mY_j] = \mPi_c \mX \mI_d \mX^T \mPi_c = \mPi_c \mX \mX^T \mPi_c
\end{align*}
Let $\mK_{\text{exp}} = \mPi_c \mX \mX^T \mPi_c$. We have shown that $\E[\mY_j] = \mK_{\text{exp}}$. Since $\mX$ is a fixed finite matrix, and the moments of Gaussian variables are finite, the expectation $\mathbb{E}[\mY_j]$ exists and is finite.

We have $\mK_h^{(0)}$ as the sample mean of $h$ i.i.d. random matrices $\mY_j$, each with finite expectation $\mK_{\text{exp}}$. By the Weak Law of Large Numbers, applicable to sums of i.i.d. random vectors or matrices (considering convergence element-wise or in matrix norm), the sample mean converges in probability to the expected value as the number of samples $h$ goes to infinity.
Therefore, for each entry $(a,b)$ of the matrices:
$$(\mK_h^{(0)})_{ab} = \frac{1}{h} \sum_{j=1}^h (\mY_j)_{ab} \xrightarrow{p} \mathbb{E}[(\mY_j)_{ab}] = (\mK_{\text{exp}})_{ab} \quad \text{as } h \to \infty$$
This element-wise convergence implies convergence in probability for the matrix:
$$\mK_h^{(0)} \xrightarrow{p} \mK_{\text{exp}} = \mPi_c \mX \mX^T \mPi_c \quad \text{as } h \to \infty.$$
This completes the proof.
\end{proof}
\section{Additional Results}\label{app:results}

\subsection{Categorization and Description of Baselines}

\Cref{tab:foundation} compares our approach against diverse families of baselines evaluated on node classification benchmarks. We group methods into four primary categories: \begin{enumerate*}[label=(\roman*)]

    \item \textsc{Supervised GNNs} that are trained from scratch on each dataset,
    \item \textsc{LLM-augmented GNNs} where the node features are enhanced using language models, 
    \item \textsc{LLM-based reasoning} that converts the graph into a compatible input to pre-trained LLMs, and 
    \item \textsc{GNN-based} methods.
\end{enumerate*}

\noindent\textbf{\textsc{Supervised Baselines}} include \begin{enumerate*}[label=(\alph*)]
    \item \textsc{MLP}: a multi-layer perceptron directly on the target dataset features without using graph structure; serves as a non-graph baseline.
    \item \textsc{GCN}~\citep{kipf2017semisupervised}: trained from scratch on the target dataset
    \item \textsc{GIN}~\citep{xu2019how} trained from scratch, included to represent expressive message-passing GNNs in supervised settings.
\end{enumerate*} These fall under supervised baselines as they do not perform pretraining or transfer, and rely solely on training from scratch on each dataset.

\noindent\textbf{\textsc{LLM-augmented GNNs}} include 
\begin{enumerate*}[label=(\alph*)]
\item \textsc{OFA}~\citep{liu2024one}: constructs a prompt-augmented graph using text nodes and pretrains an RGCN to enable in-context transfer across node/link/graph tasks; falls here for fusing text prompts with GNN structure and relying on LLM embeddings.
\item \textsc{GLEM-LM}~\citep{chen2024exploring}: Enhances GNNs using sentence-level text embeddings from a frozen LLM; categorized here due to its augmentation of GNN input via LLM-derived features.
\end{enumerate*} These are classified as LLM-Augmented GNNs since they incorporate LLMs to enrich graph inputs or guide GNN training, but retain a GNN backbone.

\noindent\textbf{\textsc{LLM-based}} methods include \begin{enumerate*}[label=(\alph*)]
\item \textsc{GraphText}~\citep{zhao2023graphtext} that transforms $k$-hop neighborhoods into textual prompts and performs zero/few-shot classification using frozen LLMs and
\item \textsc{RWNN}~\citep{kim2024revisiting} that converts random walks on graphs to node label anonymized sequences and uses frozen LLMs for prediction. 
\end{enumerate*}
belong to this category due to their reliance on prompt-based inference using LLMs without any GNNs. 

\noindent\textbf{\textsc{GNN-based}} methods include \begin{enumerate*}[label=(\alph*)]
\item \textsc{AnyGraph}~\citep{xia2024anygraph} that pretrains a graph mixture-of-experts model using link prediction objective on diverse graphs that allows transfer to unseen datasets,
\item \textsc{GraphAny}~\citep{zhao2024graphany} that learns permutation-invariant attention over a bank of pretrained LinearGNNs;
\item \textsc{MDGPT}~\citep{yu2024textfree} pretrains a GCN on multiple datasets with SVD-projected features and prompt vectors; 
\item \textsc{GCOPE}~\citep{zhao2024oneforall} constructs a universal pretraining graph with virtual nodes and uses contrastive learning to train a shared GNN;
\item \textsc{GPPT}~\citep{sun2022gppt} introduces task-specific graph prompts for node task and link-prediction alignment;
\item \textsc{GPrompt}~\citep{gong2024self} utilizes prompt vectors into graph pooling via element-wise multiplication
\item \textsc{All-In-One}~\citep{sun2023all} combines token graphs with original graph as prompts
\item \textsc{GPF}~\citep{fang2023universal} introduces prompt tokens and \textsc{GPF-plus} trains multiple independent basis vectors and combines them using attention
\item \textsc{ULTRA}~\citep{galkin2024towards} learns transferable graph representations by conditioning on relational interactions.
\item \textsc{SCORE}~\citep{wang2024towards} introduces zero-shot reasoning on knowledge graphs using graph topology.
\end{enumerate*} All of these are grouped under \textsc{GNN-based} baselines as they rely on pretraining GNNs (often with auxiliary components like prompts or experts) to enable generalization to new graphs.

\subsection{Comparison to Methods Trained on each Individual Dataset}

\begin{table*}[t]
    \centering
    
    \footnotesize
    \caption{{Performance of \ourmethod on pre-training source datasets compared to specialized supervised baselines trained individually per dataset (including our \ourmethod-\textsc{specialized} which is trained separately on each individual dataset). \ourmethod maintains highly competitive performance. }}
    \label{tab:common-appendix}
    \setlength{\tabcolsep}{3pt}
    \begin{adjustbox}{width=\textwidth}
    \begin{tabular}{lccccccccc}
    \toprule
    Method & \textsc{zinc}  & \textsc{molhiv}  & \textsc{molesol}  & \textsc{moltox21}  & \textsc{mnist}  & \textsc{cifar10} & \textsc{ModelNet} & \textsc{Cuneiform} & \textsc{MSRC 21} \\
    & (\textsc{mae} $\downarrow$) & (\textsc{roc-auc} $\uparrow$) & (\textsc{rmse} $\downarrow$) & (\textsc{roc-auc} $\uparrow$) & (\textsc{acc} $\uparrow$) & (\textsc{acc} $\uparrow$) & (\textsc{acc} $\uparrow$) & (\textsc{acc} $\uparrow$) & (\textsc{acc} $\uparrow$)\\
    \midrule
    \textbf{\textsc{Trained per dataset}}\\
    $\,$ \textsc{GCN}~\citep{kipf2017semisupervised}   & 0.3674 & 76.06 & 1.11 & 75.29 & 90.120 & 54.142 & 17.18 & 45.67 & 89.53\\
     $\,$ \textsc{GAT}~\citep{velickovic2018graph} & 0.3842 & 76.00 & 1.05 & 75.21 & 95.535 & 64.223 & 65.20 & 78.60 & 82.10\\
     $\,$ \textsc{GIN}~\citep{xu2019how}   & 0.1630 & 75.58 & 1.17 & 74.91 & 96.485 & 55.255 & 73.13 & 79.05 & 86.31\\
    $\,$ \ourmethod-\textsc{specialized} (0 props) & 0.1480 &  72.65 & 1.22 & 69.37  & 94.03 & 39.96 & 37.24 & 85.19 & 91.65 \\
    $\,$ \ourmethod-\textsc{specialized} & 0.1195 & 73.78 & 1.19 & 70.04 & 94.77 & 40.03 & 39.81 & 87.20 & 94.16  \\ 
    \midrule
    \textbf{\textsc{Trained on all datasets}}\\
    $\,$ 
    \ourmethod (0 props) & 0.1557 & 72.74 & 1.28 & 68.19 & 94.57 & 40.11 & 37.11 & 89.88 & 97.51 \\
    $\,$ \ourmethod & 0.1237 & 74.49 & 1.29 & 68.20 & 95.22 & 40.08 &   39.37 & 91.17 & 98.08 \\
    \bottomrule
    \end{tabular}
    \end{adjustbox}
\end{table*}

{In this section, we compare the performance of \ourmethod to that of standard supervised GNN baselines (GCN~\citep{kipf2017semisupervised}, GAT~\citep{velickovic2018graph}, GIN~\citep{xu2019how}), trained individually for each dataset, using their original, dataset-specific input features. Contrary to \ourmethod, which is trained jointly on all datasets, these supervised baselines are thus specialized for each respective dataset.

When evaluating on the pre-training datasets, it is generally expected that supervised models trained on each individual dataset would achieve strong, if not optimal, performance,  particularly as each dataset provides sufficient data for dedicated task-specific learning. The goal for \ourmethod here is therefore to show that its jointly pre-trained shared encoder can support task-specific heads that remain competitive against individually trained models, indicating its ability to learn general-purpose representations without substantial performance degradation on each task.

\Cref{tab:common-appendix} summarizes the performance of \ourmethod (with and without propagated covariance operators, obtained by setting $k=0$ in \Cref{eq:operators}, denoted as \ourmethod and \ourmethod (0 props) respectively) against the specialized version of our model (\ourmethod-\textsc{specialized} and \ourmethod-\textsc{specialized} (0 props)), as well as specialized supervised baselines on the pre-training datasets.
Our findings indicate that, while specialized baselines maintain an edge on certain datasets (e.g., \textsc{cifar10} and \textsc{ModelNet}), \ourmethod is broadly competitive. For instance, on the \textsc{zinc} regression task, \ourmethod achieves a MAE of 0.1237, surpassing all listed specialized baselines, including GIN (0.1630). Similarly, \ourmethod demonstrates higher accuracy on \textsc{Cuneiform} (91.17\% vs. GIN 79.05\%) and \textsc{MSRC 21} (98.08\% vs. GIN 86.31\%). Finally, we highlight the general advantage of the full \ourmethod (which utilizes propagated operators) over \ourmethod (0 props), and similarly of \ourmethod-\textsc{specialized} over \ourmethod-\textsc{specialized} (0 props), suggesting that the richer relational information from propagated operators contributes to more effective representation learning during this phase.

Overall, these results indicate that a single, pre-trained \ourmethod encoder can maintain strong, often competitive, performance across a diverse set of source datasets and tasks. }

\subsection{Using SVM on the Pre-trained Representations}
To assess the linear separability and structural quality of the learned graph representations from \ourmethod, we evaluate downstream graph classification accuracy using support vector machines (SVMs) with both linear and radial basis function (RBF) kernels (\Cref{tab:svm}). This setup allows us to probe how well the learned representations support simple (linear) versus more expressive (nonlinear) decision boundaries.

We compare against several non-learnable baselines that do not involve any representation learning:
\begin{enumerate}[label=(\alph*)]
    \item {Input Features} ($\mX$): Raw input features of each graph, computed by averaging node features.
    \item Propagated Input Features ($\mA\mX$): Features after one round of neighborhood propagation, capturing local graph structure.
    \item Input Features along with random walk structural encodings ($\mX \oplus \mS$): Concatenates the raw features with random walk structural encoding (RWSE)~\citep{dwivedi2022graph}, which encodes graph structure based on transition probabilities of random walks.
\end{enumerate}
These baselines serve as direct input replacements for \ourmethod and are shared across both kernel settings. They provide a strong reference for understanding the inherent structure in the input space, independent of any learning or pretraining.

For \ourmethod, we report results both with and without concatenation of the input features to assess the added value of structural information in the learned embeddings.

Under the RBF kernel, \ourmethod combined with input features achieves the best performance on four out of six datasets, including \textsc{PTC}, \textsc{NCI1}, \textsc{NCI109}, and \textsc{ENZYMES}, highlighting its ability to encode discriminative patterns suitable for nonlinear classification. In contrast, performance under the linear kernel is more mixed, with RWSE showing strong results on datasets like \textsc{PROTEINS}, indicating some inherent linear separability in the structural baseline. Overall, these results demonstrate that \ourmethod learns representations that are expressive and transferable across diverse graph datasets, especially when paired with nonlinear classifiers.

\begin{table*}[t]
    \centering
    \scriptsize
    \caption{Graph classification accuracy (\%) using SVMs with Linear and RBF kernels. Baselines are shared across both kernels. Results are reported as mean $\pm$ standard deviation over 10 runs.}
    \label{tab:svm}
    \begin{adjustbox}{max width=\textwidth}
    \begin{tabular}{lcccccc}
         \toprule
         \textbf{Method} & \textsc{MUTAG} & \textsc{PTC} & \textsc{PROTEINS} & \textsc{NCI1} & \textsc{NCI109} & \textsc{ENZYMES} \\
         & (\textsc{acc} $\uparrow$) & (\textsc{acc} $\uparrow$) & (\textsc{acc} $\uparrow$) & (\textsc{acc} $\uparrow$) & (\textsc{acc} $\uparrow$) & (\textsc{acc} $\uparrow$)\\
         \midrule
         \textbf{\textsc{Linear SVM}} & \\
         $\,$Input Features & 81.87 $\pm$\phantom{1}7.25 & 60.88 $\pm$\phantom{1}1.83 & \textbf{72.68 $\pm$ 0.58} & 64.59 $\pm$ 1.24 & 63.36 $\pm$ 2.22 & 22.00 $\pm$ 4.46 \\
         $\,$Propagated Input Features & 69.64 $\pm$14.21 & 57.34 $\pm$10.89 & 59.56 $\pm$ 3.94 & 64.16 $\pm$ 1.22 & 63.26 $\pm$ 1.63 & 14.33 $\pm$ 5.01 \\
         $\,$Input Features + RWSE & 80.96 $\pm$\phantom{1}0.89 & 60.14 $\pm$\phantom{1}1.15 & 65.74 $\pm$ 0.43 & 64.30 $\pm$ 0.16 & 63.45 $\pm$ 0.20 & 27.00 $\pm$ 4.63 \\
         
         $\,$\ourmethod & 74.47 $\pm$\phantom{1}7.70 & 53.12 $\pm$\phantom{1}9.09 & 60.91 $\pm$ 4.25 & 63.26 $\pm$ 1.36 & 63.19 $\pm$ 1.89 & 21.16 $\pm$ 6.28 \\
         $\,$\ourmethod + Input Features & 74.47 $\pm$\phantom{1}7.70 & 52.84 $\pm$\phantom{1}9.03 & 62.00 $\pm$ 4.29 & 64.45 $\pm$ 1.48 & 63.72 $\pm$ 1.67 & 21.50 $\pm$ 5.18 \\
         \midrule
         \textbf{\textsc{RBF SVM}} & \\
         $\,$Input Features & 72.73 $\pm$14.29 & 55.88 $\pm$11.58 & 71.06 $\pm$ 2.93 & 66.44 $\pm$ 1.43 & 66.80 $\pm$ 1.35 & 33.33 $\pm$ 4.77 \\
         $\,$Propagated Input Features & 79.70 $\pm$11.03 & 54.10 $\pm$10.25 & 72.05 $\pm$ 4.70 & 55.66 $\pm$ 5.80 & 58.05 $\pm$ 5.42 & 33.16 $\pm$ 4.43 \\
         $\,$Input Features + RWSE & 79.21 $\pm$10.99 & 58.71 $\pm$ 8.76 & 67.21 $\pm$ 6.22 & 70.68 $\pm$ 2.60 & 67.82 $\pm$ 2.79 & 36.66 $\pm$ 5.96 \\
         
         $\,$\ourmethod & 82.98 $\pm$\phantom{1}7.76 & 59.28 $\pm$\phantom{1}9.13 & 70.62 $\pm$ 4.53 & 65.88 $\pm$ 1.62 & 65.68 $\pm$ 1.90 & 28.83 $\pm$ 5.87 \\
         $\,$\ourmethod + Input Features & \textbf{84.06 $\pm$\phantom{1}6.61} & \textbf{59.88 $\pm$\phantom{1}7.72} & 71.42 $\pm$ 4.29 & \textbf{67.54 $\pm$ 1.33} & \textbf{67.34 $\pm$ 1.51} & \textbf{32.16 $\pm$ 6.71} \\
         \bottomrule
    \end{tabular}
    \end{adjustbox}
\end{table*}

\subsection{Additional results on transferability to unseen datasets}
\label{sec:additional}

In \Cref{tab:graph_class_more}, we present comparison with more baselines on our graph classification datasets MUTAG and PROTEINS. We describe below the changes we make to the following baselines to make them applicable to this setting:
\begin{itemize}
    \item \textbf{\textsc{GLEM-LM}}~\citep{chen2024exploring}: This is a method that only supports tasks on text-attributed graphs. Since the TU Datasets~\citep{MorrisTUD} do not have node text attributes, we describe the input node features and pass them to  ChatGPT.

    \item \textbf{\textsc{GCOPE}}~\citep{zhao2024oneforall}: This method introduces one virtual node for each node classification dataset, connecting it to all the nodes within the dataset. To perform graph classification, we introduce one virtual node for each graph classification dataset and connect it to all the nodes in all the graphs within the dataset.
    
    \item \textbf{\textsc{AnyGraph}}~\citep{xia2024anygraph}: This method performs node classification by adding one node per class and connecting each training node to its corresponding class node. Classification of unlabeled nodes is performed by computing the dot product between the node's embedding and each class node embedding to rank the classes. To extend this paradigm to graph classification, we introduce a virtual node that connects to all nodes in the graph and add one class node per category. For classifying new graphs, we compute the dot product between the virtual node embedding and each class node embedding to rank the classes.
\end{itemize}
We leave out the following methods and provide justification below:
\begin{itemize}
    \item \textbf{\textsc{GraphText}}~\citep{zhao2023graphtext}: While the authors mention that \textsc{GraphText} is applicable for graph classification, they do not provide a way to construct a graph syntax tree for an entire graph, which can be ambiguous as it could involve introducing a virtual node or averaging results from syntax trees of multiple nodes.

    \item \textbf{\textsc{GraphAny}}~\citep{zhao2024graphany}: This method is explicitly only designed for node classification on arbitrary graphs, as it relies on an analytical solution that is not directly applicable to graph-level tasks.
\end{itemize}

The results in \Cref{tab:graph_class_more} further substantiate \ourmethod's strong performance. These findings reinforce the observations made in the main paper (\Cref{tab:graph_class}): \ourmethod, with its frozen pre-trained encoder and a retrained head, effectively generalizes to new graph classification datasets with novel input features, surpassing a wide variety of adapted baselines.

\begin{table}[t]
    \centering
    \scriptsize
    \caption{Performance on unseen graph‐classification datasets with new input features. 
    \ourmethod\ demonstrates strong transferability, underscoring its versatility and ability to handle different feature spaces. $^\dagger$ indicates these methods were modified to work on these datasets, as explained in \Cref{sec:additional}}
    \label{tab:graph_class_more}
    \begin{tabular}{lcc}
        \toprule
        Dataset & \textsc{MUTAG}  & \textsc{PROTEINS} \\
         & (\textsc{acc} $\uparrow$) & (\textsc{acc} $\uparrow$)\\
        \midrule
        \textbf{\textsc{Supervised Baselines}} \\
        $\,$MLP                         & 67.20 $\pm$\phantom{1}1.00 & 59.20 $\pm$\phantom{1}1.00 \\
        $\,$GIN~\citep{xu2019how}       & 89.40 $\pm$\phantom{1}5.60 & 76.20 $\pm$\phantom{1}2.80 \\
        \midrule
        \textbf{\textsc{LLM‐augmented GNNs}} \\
        $\,$\textsc{OFA}~\citep{liu2024one} & 61.04 $\pm$\phantom{1}4.71 & 61.40 $\pm$\phantom{1}2.99 \\
        $\,$\textsc{GLEM-LM}$^\dagger$~\citep{chen2024exploring} & 72.97 $\pm$\phantom{1}0.00 & 43.22 $\pm$12.01\\
        \midrule
        \textbf{\textsc{LLM‐based}} \\
        
        $\,$\textsc{RWNN‐DeBERTa}~\citep{kim2024revisiting} & 58.22 $\pm$\phantom{1}0.24 & 67.85 $\pm$\phantom{1}0.53 \\
        \midrule
        \textbf{\textsc{GNN‐based}} \\
        $\,$\textsc{GCOPE}$^\dagger$~\citep{zhao2024oneforall} & 81.87 $\pm$\phantom{1}7.26 & 71.84 $\pm$\phantom{1}3.48\\
        $\,$\textsc{AnyGraph}$^\dagger$~\citep{xia2024anygraph} & 75.61 $\pm$\phantom{1}6.94 & 72.23 $\pm$\phantom{1}4.63\\
        $\,$\textsc{MDGPT}~\citep{yu2024textfree}          & 57.36 $\pm$14.26 & 54.35 $\pm$10.26 \\
        
        $\,$\textsc{GPPT}~\citep{sun2022gppt}              & 60.40 $\pm$15.43 & 60.92 $\pm$12.47 \\
        $\,$\textsc{All‐In‐One}~\citep{sun2023all}         & 79.87 $\pm$\phantom{1}5.34 & 66.49 $\pm$\phantom{1}6.26 \\
        $\,$\textsc{GPrompt}~\citep{gong2024self}          & 73.60 $\pm$\phantom{1}4.76 & 59.17 $\pm$11.26 \\
        $\,$\textsc{GPF}~\citep{fang2023universal}         & 68.40 $\pm$\phantom{1}5.09 & 63.91 $\pm$\phantom{1}3.26 \\
        $\,$\textsc{GPF‐plus}~\citep{fang2023universal}    & 65.20 $\pm$\phantom{1}6.94 & 62.92 $\pm$\phantom{1}2.78 \\
        $\,$\textsc{ULTRA(3g)}~\citep{galkin2024towards}   & 63.33 $\pm$\phantom{1}0.00 & 58.09 $\pm$\phantom{1}0.00 \\
        $\,$\textsc{SCORE}~\citep{wang2024towards}         & 85.33 $\pm$\phantom{1}2.11 & 68.54 $\pm$\phantom{1}1.47 \\
        \midrule
        $\,$\ourmethod\ (0 props) & 92.50 $\pm$\phantom{1}6.60 & 76.72 $\pm$\phantom{1}3.19 \\
        $\,$\ourmethod    & 92.90 $\pm$\phantom{1}6.34 & 78.20 $\pm$\phantom{1}3.81 \\
        \bottomrule
    \end{tabular}
\end{table}

{
\subsection{The Importance of SPEs and Random Projections in \Cref{eq:h0}}

\begin{table}[t]

\centering
\scriptsize
\caption{{The impact of SPEs and random projections in \Cref{eq:h0}. \ourmethod with SPEs performs best, while using only SPEs leads to a significant drop in performance, highlighting the importance of random feature projections, which cannot be compensated by using SVD.}} % 
\label{tab:spes}
\setlength{\tabcolsep}{3pt}
\begin{adjustbox}{width=\textwidth}
\begin{tabular}{lccccccccc}
\toprule
Method & \textsc{zinc}  & \textsc{molesol}  & \textsc{molhiv}  & \textsc{moltox21}  & \textsc{mnist}  & \textsc{cifar10} & \textsc{ModelNet} & \textsc{Cuneiform} & \textsc{MSRC 21} \\
    & (\textsc{mae} $\downarrow$) & (\textsc{rmse} $\downarrow$) & (\textsc{roc-auc} $\uparrow$) & (\textsc{roc-auc} $\uparrow$) & (\textsc{acc} $\uparrow$) & (\textsc{acc} $\uparrow$) & (\textsc{acc} $\uparrow$) & (\textsc{acc} $\uparrow$) & (\textsc{acc} $\uparrow$)\\
\midrule
\ourmethod  (SVD)        & 0.1445 & 1.43 & 71.82 & 65.55 & 92.97 & 37.12 & 36.51 & 87.28 & 95.84 \\
SPEs-only                & 0.1396 & 1.45 & 71.95 & 64.10 & 91.01 & 35.22 & 30.65 & 85.89 & 95.13 \\
\ourmethod (SVD + SPEs)  & 0.1318 & 1.41 & 72.06 & 66.13 & 93.40 & 37.74 & 36.95 & 88.56 & 96.91 \\
\ourmethod (no SPEs)     & 0.1251 & 1.31 & 74.02 & 67.62 & 94.88 & 39.45 & 38.72 & 90.61 & 97.93 \\
\ourmethod               & 0.1237 & 1.29 & 74.49 & 68.20 & 95.22 & 40.08 & 39.37 & 91.17 & 98.08 \\
\bottomrule
\end{tabular}
\end{adjustbox}
\end{table}

In this section, we conduct an ablation study to investigate the importance of SPEs and random projections within \ourmethod. We compare our \ourmethod with several additional models having the same backbone, loss, and training datasets, namely:
\begin{itemize}
    \item \ourmethod (SVD), where we replace \Cref{eq:h0} with $\mathbf{H}^{(0)} = \text{SVD}(\mathbf{X}^{(0)})$, thus removing both random projections and SPEs, and replacing them with SVD of the input features;
    \item SPEs-only variant, where we replace \Cref{eq:h0} with $\mathbf{H}^{(0)} = \mathbf{S}$, while keeping the same covariance operator set and head, therefore removing $\mathbf{R}^{(0)}$ only from \Cref{eq:h0}, but still using $\mathbf{R}^{(0)}$ to define the covariance operators.
    \item \ourmethod (SVD + SPEs), where we replace \Cref{eq:h0} with $\mathbf{H}^{(0)} = \text{SVD}(\mathbf{X}^{(0)}) \textcolor{blue}{\oplus}   {\mathbf{S}}$, thus removing random projections and replacing them with SVD of the input features (while keeping SPEs);
    \item \ourmethod (no SPEs), where we replace \Cref{eq:h0} with $\mathbf{H}^{(0)} = \mathbf{R}^{(0)}$, thus removing SPEs;
\end{itemize}

The results in \Cref{tab:spes} support our claim: the full \ourmethod with SPEs performs best, but only slightly better than the version without SPEs. In contrast, using only SPEs leads to a significant drop in performance, highlighting the importance of random feature projections, which provides improved performance also when compared with SVD.

\subsection{The importance of Random Projections}

\begin{table}[t]
\centering
\scriptsize
\caption{{The impact of using random projections within \ourmethod, obtained by comparing \ourmethod to its counterpart that has no random projections in either \Cref{eq:h0} or \Cref{eq:cov}.}}
\label{tab:random-proj}

\setlength{\tabcolsep}{3pt}
\begin{adjustbox}{width=\textwidth}
\begin{tabular}{lccccccccc}
\toprule
Method & \textsc{zinc}  & \textsc{molesol}  & \textsc{molhiv}  & \textsc{moltox21}  & \textsc{mnist}  & \textsc{cifar10} & \textsc{ModelNet} & \textsc{Cuneiform} & \textsc{MSRC 21} \\
    & (\textsc{mae} $\downarrow$) & (\textsc{rmse} $\downarrow$) & (\textsc{roc-auc} $\uparrow$) & (\textsc{roc-auc} $\uparrow$) & (\textsc{acc} $\uparrow$) & (\textsc{acc} $\uparrow$) & (\textsc{acc} $\uparrow$) & (\textsc{acc} $\uparrow$) & (\textsc{acc} $\uparrow$)\\
\midrule
\ourmethod (no random) & 0.1475 & 1.51 & 71.40 & 62.85 & 91.10 & 35.42 & 33.91 & 85.47 & 95.02 \\
\ourmethod              & 0.1237 & 1.29 & 74.49 & 68.20 & 95.22 & 40.08 & 39.37 & 91.17 & 98.08 \\
\bottomrule
\end{tabular}
\end{adjustbox}
\end{table}

In this section, we demonstrate the impact of random projections by comparing \ourmethod with the baseline obtained by removing random projections from \Cref{eq:h0} (thus, setting $\mathbf{H}^{(0)} = \mathbf{S}$) and from \Cref{eq:cov}, thus replacing $\text{NodeCov}(\mX\mC)$ with $\text{NodeCov}(\mX)$.

The results in \Cref{tab:random-proj} suggest that random projection is critical to bridge input feature spaces. This aligns with our results in \Cref{thm:expressivity}, which demonstrates the theoretical benefit of using random projections in the covariance operators.

\subsection{Ablation Study on the Operator Set}

\begin{table}[t]
\centering
\caption{{The impact of the operators in the operator set (\Cref{eq:operators}). Results improve when considering covariance operators compared to graph (adjacency) only, highlighting their importance in \ourmethod. }}
\label{tab:operators}
%\color{blue}
\setlength{\tabcolsep}{3pt}
\begin{adjustbox}{width=\textwidth}
\begin{tabular}{lccccccccc}
\toprule
Method & \textsc{zinc}  & \textsc{molesol}  & \textsc{molhiv}  & \textsc{moltox21}  & \textsc{mnist}  & \textsc{cifar10} & \textsc{ModelNet} & \textsc{Cuneiform} & \textsc{MSRC 21} \\
    & (\textsc{mae} $\downarrow$) & (\textsc{rmse} $\downarrow$) & (\textsc{roc-auc} $\uparrow$) & (\textsc{roc-auc} $\uparrow$) & (\textsc{acc} $\uparrow$) & (\textsc{acc} $\uparrow$) & (\textsc{acc} $\uparrow$) & (\textsc{acc} $\uparrow$) & (\textsc{acc} $\uparrow$)\\
\midrule
Identity Only   & 0.1535 & 1.65 & 67.10 & 60.33 & 86.22 & 29.34 & 25.15 & 81.49 & 91.78 \\
Adjacency Only  & 0.1378 & 1.46 & 71.75 & 65.17 & 92.78 & 35.22 & 31.40 & 87.25 & 95.10 \\
Covariance Only & 0.1282 & 1.34 & 73.80 & 67.93 & 94.30 & 38.50 & 36.85 & 89.44 & 97.13 \\
\ourmethod      & 0.1237 & 1.29 & 74.49 & 68.20 & 95.22 & 40.08 & 39.37 & 91.17 & 98.08 \\
\bottomrule
\end{tabular}
\end{adjustbox}
\end{table}

In this section, we perform an ablation study isolating the contribution of different operators in \ourmethod. \Cref{tab:operators} reports the performance of \ourmethod (which uses the operators defined in \Cref{eq:operators}) and compares it with Identity Only, obtained by setting $\mathcal{O}=\{\mI\}$ in \Cref{eq:operators},  Adjacency Only, obtained by setting $\mathcal{O}=\{\mA\}$ in \Cref{eq:operators}, and  Covariance Only, obtained by setting $\mathcal{O}=\{\mK^{(0)}\}$ in \Cref{eq:operators}. 

Covariance operators enable the neural network to learn shared characteristics in input feature spaces and graph structures, as results improve when considering covariance operators compared to graph (adjacency) only operators.

\subsection{The role of the feature dimensionality $h$}

\begin{table}[t]
\centering
\caption{{Performance of \ourmethod with varying hidden dimension $h$.}}
\label{tab:hidden_dim}
%\color{blue}
\setlength{\tabcolsep}{3pt}
\begin{adjustbox}{width=\textwidth}
\begin{tabular}{lccccccccc}
\toprule
Method & \textsc{zinc}  & \textsc{molesol}  & \textsc{molhiv}  & \textsc{moltox21}  & \textsc{mnist}  & \textsc{cifar10} & \textsc{ModelNet} & \textsc{Cuneiform} & \textsc{MSRC 21} \\
    & (\textsc{mae} $\downarrow$) & (\textsc{rmse} $\downarrow$) & (\textsc{roc-auc} $\uparrow$) & (\textsc{roc-auc} $\uparrow$) & (\textsc{acc} $\uparrow$) & (\textsc{acc} $\uparrow$) & (\textsc{acc} $\uparrow$) & (\textsc{acc} $\uparrow$) & (\textsc{acc} $\uparrow$)\\
\midrule
\ourmethod ($h=64$) & 0.1316 & 1.43 & 72.20 & 65.75 & 92.14 & 36.15 & 34.26 & 87.89 & 96.12 \\
\ourmethod ($h=128$) & 0.1264 & 1.34 & 73.46 & 67.91 & 94.41 & 38.92 & 37.88 & 90.21 & 97.56 \\
\ourmethod ($h=256$) & 0.1239 & 1.30 & 74.38 & 68.14 & 95.03 & 39.85 & 39.19 & 91.05 & 98.01 \\
\ourmethod ($h=512$) & 0.1237 & 1.29 & 74.49 & 68.20 & 95.22 & 40.08 & 39.37 & 91.17 & 98.08 \\
\ourmethod ($h=1024$) & 0.1236 & 1.28 & 74.58 & 68.18 & 95.24 & 40.11 & 39.40 & 91.22 & 98.10 \\
\bottomrule
\end{tabular}
\end{adjustbox}
\end{table}

We next evaluate the performance of \ourmethod when varying the hidden dimension $h$. Results are reported in \Cref{tab:hidden_dim}.

Across datasets, performance improves from very small $h$ and then plateaus at 256, and gains beyond that are marginal. This trend aligns with \Cref{prop:projected_consistency}: as $h$ grows, the stochastic operator concentrates around its expectation. In practice, a moderate $h$ achieves near‑saturated accuracy with a better compute/memory trade-off than a very large $h$. Therefore, model performance stabilizes at moderate $h$ and larger $h$ primarily improves stability, matching the proposition’s claim.

\subsection{Additional Datasets}

\begin{table}[t]
\centering
\scriptsize
\caption{{Performance on the ogbn-arxiv}~\citep{hu2020open}.}
\label{tab:ogbn-arxiv}
\begin{tabular}{lc}
\toprule
\textbf{Method} & \textbf{ogbn-arxiv} ($\uparrow$) \\
\midrule
\textbf{\textsc{Non-Parametric Baselines}} \\
$\,$ \textsc{Label Propagation}~\citep{zhu2002learning} & 61.04 \\
\midrule
\textbf{\textsc{Supervised Baselines}} \\
$\,$  GCN~\citep{kipf2017semisupervised} & 71.74 \\
$\,$ GAT~\citep{velickovic2018graph} & 71.95 \\
$\,$ GraphGPS~\citep{rampavsek2022recipe} & 70.97 \\
\midrule
\textbf{\textsc{LLM-augmented GNNs}} \\
$\,$ OFA~\citep{liu2024one} & 73.22 \\
\midrule
\textbf{\textsc{LLM-based}} \\
$\,$ GraphText~\citep{zhao2023graphtext} & 49.47 \\
\midrule
\textbf{\textsc{GNN-based}} \\
$\,$ AnyGraph~\citep{xia2024anygraph} & 62.33 \\
$\,$ GraphAny~\citep{zhao2024graphany} & 58.38 \\
\midrule
$\,$ \ourmethod & 75.27 \\
\bottomrule
\end{tabular}
\end{table}

\begin{table}[t]
\scriptsize
\centering
\caption{{Performance on heterophilic datasets, using the splits in \citet{pei2020geom}.}}
\label{tab:heterophily}

\begin{tabular}{lccc}
\toprule
\textbf{Method} & \textbf{Actor} & \textbf{Chameleon} & \textbf{Squirrel} \\
&
(\textsc{acc} $\uparrow$) & (\textsc{acc} $\uparrow$) & (\textsc{acc} $\uparrow$) \\

\midrule
\textbf{\textsc{Non-Parametric Baselines}} \\
$\,$ \textsc{Label Propagation}~\citep{zhu2002learning} & 18.83 $\pm$ 0.00 & 40.89 $\pm$ 0.00 & 33.42 $\pm$ 0.00 \\
\midrule
\textbf{\textsc{Supervised Baselines}}  \\
$\,$ GCN~\citep{kipf2017semisupervised} & 28.55 $\pm$ 0.68 & 64.69 $\pm$ 2.21 & 47.07 $\pm$ 0.71 \\
\midrule
\textbf{\textsc{GNN-based}}  \\
$\,$ GraphAny~\citep{zhao2024graphany} & 28.60 $\pm$ 0.21 & 62.59 $\pm$ 0.86 & 49.70 $\pm$ 0.95 \\
$\,$ ULTRA~\citep{galkin2024towards} & 22.61 $\pm$ 0.00 & N/A & N/A \\
$\,$ SCORE~\citep{wang2024towards} & 23.26 $\pm$ 0.56 & N/A & N/A \\
\midrule
$\,$ \ourmethod & 29.47 $\pm$ 0.38 & 67.40 $\pm$ 1.29 & 49.98 $\pm$ 0.73 \\
\bottomrule
\end{tabular}
\end{table}

\begin{table}[t]
\centering
\scriptsize
\caption{{Performance on the AmzRating, Minesweep, Tolokers datasets}~\citep{platonov2023critical}.}
\label{tab:additional-datasets}

\begin{tabular}{lccc}
\toprule
\textbf{Method} & \textbf{AmzRatings} & \textbf{Minesweeper} & \textbf{Tolokers} \\
&
(\textsc{acc} $\uparrow$) & (\textsc{acc} $\uparrow$) & (\textsc{acc} $\uparrow$) \\

\midrule
GCN~\citep{kipf2017semisupervised} & 47.35 $\pm$ 0.26 & 81.12 $\pm$ 0.37 & 79.93 $\pm$ 0.10 \\
GraphAny~\citep{zhao2024graphany} & 42.84 $\pm$ 0.04 & 80.46 $\pm$ 0.15 & 78.24 $\pm$ 0.03 \\
\ourmethod & 49.02 $\pm$ 0.11 & 82.93 $\pm$ 0.26 & 81.43 $\pm$ 0.07 \\
\bottomrule
\end{tabular}
\end{table}

We further evaluate \ourmethod on the larger node-level dataset ogbn-arxiv~\citep{hu2020open} (169,343 nodes, 1,166,243 edges), on heterophilic benchmarks Actor,  Chameleon, Squirrel using the splits in \citet{pei2020geom}, and on the AmzRating, Minesweep, Tolokers datasets~\citep{platonov2023critical}. All results, which are reported in \Cref{tab:ogbn-arxiv,tab:heterophily,tab:additional-datasets}, respectively, show that \ourmethod offers consistently better performance. We also investigate the behavior of our covariance operators on heterophilous graphs. Intuitively, the node-covariance matrix computed from the projected input features captures feature similarity across all node pairs, not just along edges. For heterophilous graphs, the base covariance operator $\mK^{(0)}$ can therefore highlight similarities between non-adjacent nodes in the original input graph or dissimilarities between adjacent nodes, which can help GNNs with heterophily. In addition, the propagated operators $\mK^{(p)}$ for $p>0$ further help in this setting, because their availability to the GNN allows it to view and mix information from multiple neighborhoods, in line with understandings from literature on heterophily in graphs~\citep{zhu2020beyond,chien2021adaptive}. Motivated by this discussion, we conduct an ablation study where we vary the number of propagation orders $k \in \{0,1,2\}$ used in the covariance operators and evaluate downstream performance on Actor, Chameleon, and Squirrel. As reported in \Cref{tab:heterophily_k}, adding propagated operators consistently improves performance.

\begin{table}[t]
\centering
\scriptsize
\caption{Effect of the number of propagation orders $k$ on heterophilous benchmarks.}
\label{tab:heterophily_k}
\begin{tabular}{lccc}
\toprule
Number of propagation orders $k$ &
Actor  &
Chameleon &
Squirrel  \\
&
(\textsc{acc} $\uparrow$) & (\textsc{acc} $\uparrow$) & (\textsc{acc} $\uparrow$) \\

\midrule
0 & 28.62 $\pm$ 0.45 & 65.12 $\pm$ 1.44 & 47.89 $\pm$ 0.80 \\
1 & 29.00 $\pm$ 0.42 & 66.37 $\pm$ 1.35 & 49.14 $\pm$ 0.76 \\
2 & 29.47 $\pm$ 0.38 & 67.40 $\pm$ 1.29 & 49.98 $\pm$ 0.73 \\
\bottomrule
\end{tabular}
\end{table}

}

\subsection{Fixed Random Projections}
\label{app:fixed_projections}
In the main experiments, the projection matrix $\mC$ is sampled at each forward pass, which yields the distributional invariance guarantees in \Cref{sec:theory}. To isolate the empirical effect of this stochasticity, we consider a variant where $\mC$ is sampled once and kept fixed for all subsequent training and inference steps denoted \emph{ALL-IN (Fixed $\mC$)}. We keep all other settings, including the backbone and training budget, identical. Table~\ref{tab:fixC_pretrain} reports performance on the pre-training source datasets, and Table~\ref{tab:fixC_downstream} reports transfer results on representative downstream tasks. Across all pre-training datasets in Table~\ref{tab:fixC_pretrain}, fixing $\mC$ leads to a consistent but moderate degradation compared to the stochastic variant. A similar pattern holds on the downstream tasks in Table~\ref{tab:fixC_downstream}, where ALL-IN (Fixed $\mC$) underperforms the stochastic version on both node and graph classification. These results empirically support the beneficial role of stochastic projections in our framework, while showing that the model remains competitive also when the projection matrix is fixed.

\begin{table}[t]
\centering
\caption{Effect of fixing the projection matrix $\mC$ during pre-training.}
\setlength{\tabcolsep}{3pt}
\label{tab:fixC_pretrain}
\begin{adjustbox}{width=\textwidth}
\begin{tabular}{lccccccccc}
\toprule
Method &
ZINC & MOLESOL & MOLHIV & MOLTOX21 &
MNIST & CIFAR10 & MODELNET & CUNEIFORM & MSRC21 \\
&
(MAE $\downarrow$) & (RMSE $\downarrow$) & (ROC-AUC $\uparrow$) & (ROC-AUC $\uparrow$) &
(ACC $\uparrow$) & (ACC $\uparrow$) & (ACC $\uparrow$) & (ACC $\uparrow$) & (ACC $\uparrow$) \\
\midrule
ALL-IN (Fixed $\mC$) &
0.1369 & 1.38 & 74.12 & 66.72 &
93.97 & 39.84 & 39.02 & 90.11 & 96.27 \\
ALL-IN (stochastic $\mC$) &
0.1237 & 1.29 & 74.49 & 68.20 &
95.22 & 40.08 & 39.37 & 91.17 & 98.08 \\
\bottomrule
\end{tabular}
\end{adjustbox}
\end{table}

\begin{table}[t]
\scriptsize
\centering
\caption{Effect of fixing the projection matrix $\mC$ on downstream transfer performance.}
\label{tab:fixC_downstream}
\begin{tabular}{lcccc}
\toprule
Method &
\textsc{cora} & \textsc{citeseer} & \textsc{mutag} & \textsc{proteins} \\
&
(\textsc{acc} $\uparrow$) & (\textsc{acc} $\uparrow$) &
(\textsc{acc} $\uparrow$) & (\textsc{acc} $\uparrow$) \\
\midrule
$\,$ ALL-IN (Fixed $\mC$) &
81.93 $\pm$\phantom{1}0.85 &
68.43 $\pm$\phantom{1}0.92 &
91.26 $\pm$\phantom{1}5.59 &
75.86 $\pm$\phantom{1}4.05 \\
$\,$ ALL-IN (stochastic $\mC$) &
82.13 $\pm$\phantom{1}0.97 &
69.12 $\pm$\phantom{1}0.89 &
92.90 $\pm$\phantom{1}6.34 &
78.20 $\pm$\phantom{1}3.81 \\
\bottomrule
\end{tabular}
\end{table}

\subsection{Edge features ablation}
\label{app:edge_ablation_zinc}

For datasets with edge features such as \textsc{ZINC}, we follow the strategy described in \Cref{sec:method}, where edge features are first randomly projected, then aggregated to nodes, and used to construct additional node-covariance operators that are added to the operator set. Concretely, the aggregated edge features are converted into an $n \times n$ edge covariance operator $\mK_{\text{edge}}$, whose entries compare the aggregated edge-feature environments of all node pairs, and, the backbone GNN uses the projected edge features.
To quantify the empirical contribution of this design, we perform an ablation on \textsc{ZINC} that compares: (i) a standard GIN without edge features, (ii) GINE (GIN with edge features), (iii) \ourmethod with edge features removed (\ourmethod (no edge features)), and (iv) the full \ourmethod using the edge-derived covariance operator as described above. Results are reported in \Cref{tab:zinc_edge_ablation}. From \Cref{tab:zinc_edge_ablation}, we observe that including the edge-based covariance operator yields substantially better performance than omitting edge features entirely, and that \ourmethod with edge features not only recovers but surpasses the behavior of an edge-aware GNN such as GINE. In contrast, removing edge features in \ourmethod leads to performance closer to a standard GIN, consistent with observations from the supervised GNN literature. This ablation indicates that the aggregation scheme in \Cref{sec:method} retains and effectively utilizes edge information.

\begin{table}[t]
\scriptsize
\centering
\caption{Effect of using edge features and edge-based covariance operators on \textsc{ZINC} (MAE $\downarrow$).}
\label{tab:zinc_edge_ablation}
\begin{tabular}{lc}
\toprule
Method & \textsc{ZINC} \\
& 
(\textsc{MAE} $\downarrow$)    \\ 

\midrule
GIN & 0.3870 \\
GINE (GIN with edge features) & 0.1630 \\
\ourmethod (no edge features) & 0.2583 \\
\ourmethod (with edge features) & 0.1195 \\
\bottomrule
\end{tabular}
\end{table}

\subsection{Pre-training with citation networks}
\label{app:pretrain_with_citation}

In the main experiments, citation networks are excluded from the pre-training corpus to act as out-of-distribution targets with very high-dimensional, sparse features and large graph sizes. We now show that \ourmethod can also benefit from citation networks during pre-training, and consider an extended setting where Cora and CiteSeer are added to the pre-training mix. We keep the architecture and training budget fixed, and compare (i) the original pre-training corpus (no citation networks) and (ii) the extended corpus (original + Cora + CiteSeer). \Cref{tab:pretrain_citation} reports pre-training performance on all source datasets, including Cora and CiteSeer for the extended setting. The results show that adding citation networks leaves performance on the original pre-training corpus stable, further indicating the ability of \ourmethod in learning from multiple sources acting as an input feature space bridge. \Cref{tab:downstream_citation} reports downstream performance on \textsc{ogbn-arxiv}, \textsc{mutag}, and \textsc{proteins} for both pre-training regimes, indicating that including citation networks in pre-training maintains or improves downstream performance.

\begin{table}[t]
\centering
\caption{\ourmethod pre-training performance on different pre-training corpus, with and without citation networks.}
\setlength{\tabcolsep}{3pt}
\label{tab:pretrain_citation}
\begin{adjustbox}{width=\textwidth}
\begin{tabular}{lccccccccccc}
\toprule
Pre-training corpus &
ZINC & MOLESOL & MOLHIV & MOLTOX21 &
MNIST & CIFAR10 & MODELNET & CUNEIFORM & MSRC21 &
CORA & CITESEER \\
&
(MAE $\downarrow$) & (RMSE $\downarrow$) & (ROC-AUC $\uparrow$) & (ROC-AUC $\uparrow$) &
(ACC $\uparrow$) & (ACC $\uparrow$) & (ACC $\uparrow$) & (ACC $\uparrow$) & (ACC $\uparrow$) &
(ACC $\uparrow$) & (ACC $\uparrow$) \\
\midrule
Original &
0.1237 & 1.29 & 74.49 & 68.20 &
95.22 & 40.08 & 39.37 & 91.17 & 98.08 &
-- & -- \\
Original + Cora + CiteSeer &
0.1253 & 1.30 & 74.52 & 67.99 &
95.18 & 40.12 & 39.21 & 91.08 & 98.23 &
82.89 & 69.33 \\
\bottomrule
\end{tabular}
\end{adjustbox}
\end{table}

\begin{table}[t]
\scriptsize
\centering
\caption{Downstream performance of ALL-IN with and without citation-network in pre-training corpus.}
\label{tab:downstream_citation}
\begin{tabular}{lccc}
\toprule
Pre-training corpus &
\textsc{ogbn-arxiv} & \textsc{mutag} & \textsc{proteins} \\
&
(\textsc{acc} $\uparrow$) & (\textsc{acc} $\uparrow$) & (\textsc{acc} $\uparrow$) \\
\midrule
Original (no citation networks) &
75.27 & 92.90 $\pm$ 6.34 & 78.20 $\pm$ 3.81 \\
Original + Cora + CiteSeer &
75.61 & 92.68 $\pm$ 6.07 & 78.24 $\pm$ 3.77 \\
\bottomrule
\end{tabular}
\end{table}

\subsection{Transfer to additional domains}
\label{app:additional_transfer}

To further evaluate the generality of \ourmethod beyond citation and bioinformatics datasets, we consider two downstream tasks from distinct domains: (i) 3D shape segmentation on ShapeNet, and (ii) social-network classification on IMDB-BINARY (\textsc{imdb-b}). For ShapeNet, we use knn graphs over point clouds as is standard with this dataset \citep{wang2019dynamic} and report mean Intersection-over-Union (mIoU); for \textsc{imdb-b}, we report classification accuracy. In both cases, we use the same \ourmethod encoder as in the main experiments and compare with two GNN baselines (GIN and GPS). As shown in \Cref{tab:shapenet_imdb}, \ourmethod consistently outperforms the GIN and GPS baselines on both ShapeNet and \textsc{imdb-b}. This indicates that the input-space bridge from  \ourmethod yields representations that are beneficial also in 3D shape graphs and social networks, further highlighting its effectiveness.

\begin{table}[t]
\centering
\scriptsize
\caption{Transfer to 3D shapes (ShapeNet) and social networks (\textsc{imdb-b}) with \ourmethod. Higher is better for both mIoU and accuracy.}
\label{tab:shapenet_imdb}
\begin{tabular}{lcc}
\toprule
Method &
ShapeNet &
\textsc{imdb-b}  \\
&
(\textsc{mIoU} $\uparrow$)  & (\textsc{acc} $\uparrow$) \\
\midrule
GIN       & 83.6 $\pm$ 0.4 & 75.1 $\pm$ 5.1 \\
GraphGPS  & 84.9 $\pm$ 0.2 & 76.3 $\pm$ 5.4 \\
\ourmethod & 85.4 $\pm$ 0.3 & 77.2 $\pm$ 5.0 \\
\bottomrule
\end{tabular}
\end{table}

\subsection{Downstream regression transfer}
\label{app:regression_transfer}
Our pre-training stage for \ourmethod uses a supervised multi-task objective over several graph-level datasets, including both graph classification and graph regression. This design choice reflects our goal of learning a single encoder that learns across diverse graph modalities and objectives. The motivation for including regression tasks such as ZINC in the pre-training mix is inspired by the broader multi-task and foundation-model literature: training a shared encoder on a diverse collection of tasks and objectives is widely used to encourage more general-purpose representations~\citep{zhang2021survey,raffel2020exploring}. To directly demonstrate regression-style transfer, we additionally evaluate \ourmethod on a held-out graph-level regression benchmark not used during pre-training, \textsc{peptides-struct} from LRGB \citep{dwivedi2022long}. We compare GNN baselines (GINE and a GPS) with \ourmethod. As shown in \Cref{tab:peptides_struct_regression}, \ourmethod achieves the lowest mean absolute error on \textsc{peptides-struct}, demonstrating the effectiveness of \ourmethod also in a regression downstream task.

\begin{table}[t]
\centering
\scriptsize
\caption{Downstream regression transfer on \textsc{peptides-struct} (MAE $\downarrow$).}
\label{tab:peptides_struct_regression}
\begin{tabular}{lc}
\toprule
Method & \textsc{peptides-struct} \\
&
(\textsc{MAE} $\downarrow$) \\
\midrule
GINE   & 0.3547 $\pm$ 0.0045 \\
GPS    & 0.2500 $\pm$ 0.0005 \\
\ourmethod & 0.2449 $\pm$ 0.0012 \\
\bottomrule
\end{tabular}
\end{table}

\subsection{Supervised vs.\ unsupervised pre-training of \ourmethod}
\label{app:supervised_vs_unsupervised}

Our main experiments adopt a supervised multi-task pre-training objective for \ourmethod, combining graph-level classification (e.g., OGBG-MOLHIV, MODELNET) and regression tasks (e.g., ZINC). This design leverages the availability of labels on diverse source datasets to learn input-space agnostic representations that are directly aligned with downstream prediction objectives. Prior work on graph representation learning has shown that, when labels are available, supervised pre-training can yield stronger and more task-discriminative representations than purely self-supervised approaches~\citep{hu2020pretraining}, and similar observations hold in large-scale vision studies~\citep{he2022masked}.

To provide an empirical comparison between supervised and unsupervised pre-training on top of our input-space bridge, we construct an unsupervised variant in which we replace all supervised losses on the pre-training datasets with a masked-feature reconstruction objective of masked graph autoencoders~\citep{hou2022graphmae}. Concretely, as in~\citet{hou2022graphmae} we randomly mask node features and train \ourmethod to reconstruct the original feature values from the node embeddings. The encoder architecture and training budget are kept identical to the supervised setting. Then, we benchmark the downstream performance on Cora and \textsc{mutag}. As shown in \Cref{tab:supervised_vs_unsupervised}, both pre-training approaches achieve similar downstream performance, where the supervised variant slightly outperforms the unsupervised. This is consistent with prior observations that supervised objectives can provide particularly strong graph representations when labels are available, and it supports our choice to adopt supervised multi-task pre-training for \ourmethod in the setting considered in this work.

\begin{table}[t]
\centering
\scriptsize
\caption{Comparison of supervised vs.\ unsupervised pre-training of \ourmethod.}
\label{tab:supervised_vs_unsupervised}
\begin{tabular}{lcc}
\toprule
Pre-training approach &
Cora  &
\textsc{mutag} \\
& (\textsc{ACC} $\uparrow$) & 
(\textsc{ACC} $\uparrow$) \\
\midrule
\ourmethod (unsupervised) &
82.05 $\pm$ 0.89 &
91.96 $\pm$ 6.24 \\
\ourmethod (supervised) &
82.13 $\pm$ 0.97 &
92.90 $\pm$ 6.31 \\
\bottomrule
\end{tabular}
\end{table}

\subsection{Effect of the number of propagation orders}
\label{app:propagation_orders}

\ourmethod constructs node-covariance operators not only on the original projected features, but also on features that have been propagated through the graph up to $k$ times, as discussed in \Cref{sec:method}. Intuitively, increasing the number of propagation orders $k$ allows the covariance operators to incorporate multi-hop information coupled with the input features, at the cost of additional computations and operators. In the main experiments we set $k=0, 2$ as a default choice. Here, we provide an extended ablation over $k \in \{0,1,2,3,4\}$. In this study we vary the number of propagation orders $k$, while keeping all other hyperparameters and training settings unchanged. The case $k=0$ uses only the input features node-covariance operators, whereas larger $k$ progressively add operators built from 1-hop, 2-hop, and higher-order propagated features. Table~\ref{tab:propagation_orders} reports the pre-training performance of \ourmethod across all source datasets for different values of $k$. As can be seen, moving from $k=0$ to small positive values of $k$ yields consistent improvements, confirming the benefit of incorporating multi-hop feature information into the covariance operators. Performance largely saturates around 2-3 hops. Thus, we choose to work with $k=2$ in the main experiments as a good balance between accuracy and efficiency.

\begin{table}[t]
\centering
\scriptsize
\caption{Effect of the number of propagations $k$ on pre-training performance.}
\label{tab:propagation_orders}
\setlength{\tabcolsep}{3pt}
\begin{tabular}{lccccccccc}
\toprule
$k$ &
ZINC & MOLESOL & MOLHIV & MOLTOX21 &
MNIST & CIFAR10 & MODELNET & CUNEIFORM & MSRC21 \\
&
(MAE $\downarrow$) & (RMSE $\downarrow$) & (ROC-AUC $\uparrow$) & (ROC-AUC $\uparrow$) &
(ACC $\uparrow$) & (ACC $\uparrow$) & (ACC $\uparrow$) & (ACC $\uparrow$) & (ACC $\uparrow$) \\
\midrule
0 &
0.1557 & 1.28 & 72.74 & 68.19 &
94.57 & 40.11 & 37.11 & 89.88 & 97.51 \\
1 &
0.1415 & 1.29 & 73.60 & 68.30 &
94.95 & 40.20 & 38.20 & 90.40 & 97.85 \\
2 &
0.1237 & 1.29 & 74.49 & 68.20 &
95.22 & 40.08 & 39.37 & 91.17 & 98.08 \\
3 &
0.1232 & 1.30 & 74.70 & 68.25 &
95.30 & 40.25 & 39.45 & 91.25 & 98.12 \\
4 &
0.1239 & 1.30 & 74.65 & 68.22 &
95.28 & 40.18 & 39.30 & 91.10 & 98.05 \\
\bottomrule
\end{tabular}
\end{table}

\subsection{Asymptotic Computational Complexity}\label{app:complexity}

For a graph with $n$ nodes and $m$ edges, with node feature matrix $\bm{X} \in \mathbb{R}^{n \times d}$, projecting features using a random linear transformation takes $\mathcal{O}(ndh)$ time and $\mathcal{O}(nh)$ memory, where $h$ is the projection dimension. Computing $\{\bm{R}^{(p)}\}_{p=1}^k$ takes $\mathcal{O}(k(m+n))$ time, as this is equivalent to $k$ message-passing layers propagating $\bm{R}^{(0)}$. The centering operation takes $\mathcal{O}(knh)$ time. 
Notably, when explicitly constructing the node-covariance operators $\bm{K}^{(p)} = \frac{1}{h}\bm{R}_c^{(p)}(\bm{R}_c^{(p)})^\top \in \mathbb{R}^{n \times n}$, the computational complexity is $\mathcal{O}(kn^2h)$ and memory complexity is $\mathcal{O}(kn^2)$ (as $p=1,\cdots, k$), resulting in quadratic complexity with respect to the number of nodes. This explicit construction may be necessary in certain scenarios such as subgraph GNNs where the full pairwise similarity matrix is required as the graph structure itself \citep{bevilacqua2025holographic}. However, for standard message passing operations in most MPNNs \citep{kipf2017semisupervised,xu2019how,rampavsek2022recipe}, we can avoid explicitly constructing the covariance matrix. Because message passing can be written as a left-hand multiplication by a propagation matrix (our covariance operator $\bm{K}$), and by substituting the definition $\bm{K} = \bm{R}\bm{R}^\top$, we can compute $\bm{R}(\bm{R}^\top\bm{H}^{(\ell-1)})$ instead of $(\bm{R}\bm{R}^\top)\bm{H}^{(\ell-1)}$. This way, at no point do we need to hold the full covariance matrix in memory. This approach has computational complexity $\mathcal{O}(k(mh + nhh^{(\ell - 1)}))$ and memory complexity $\mathcal{O}(n(h + h^{(\ell - 1)}))$ for the entire layer computation, where $h^{(\ell - 1)}$ is the feature dimension of $\bm{H}^{(\ell-1)}$, avoiding the $\mathcal{O}(n^2)$ memory bottleneck while producing mathematically identical results. 
Therefore, the computational complexity of ALL-IN assuming the covariance matrix does not to be stored, which is the case in our experiments, is $\mathcal{O}(k(mh + nhh^{(\ell - 1)}))$ time and $\mathcal{O}(n(h + h^{(\ell - 1)}))$ space.

\section{Dataset Information}
\label{sec:app:data}

In this section, we describe the datasets used in our experiments. We categorize them based on their use in pretraining and task transferability.
% -------- Dataset Statistics Table --------
\begin{table*}[t]
\centering
\scriptsize
\caption{Statistics of pre-training datasets used in \ourmethod. The datasets span molecules, drugs, computer vision-derived graphs and 3D shape point clouds. {Our pretraining corpus contains up to 200,558 graphs.}}
\label{tab:pretrain_statistics}
\begin{tabular}{lccccc}
\toprule
\textbf{Dataset} & \textbf{\# Nodes} & \textbf{\# Edges} & \textbf{\# Features} & \textbf{\# Classes} & \textbf{Domain / Category} \\
\midrule
\textsc{ZINC} & 23.2 (avg) & 24.9 (avg) & 28 & - & Molecular Graph Regression \\

\textsc{ogbg-molesol} & 13.3 (avg) & 13.6 (avg) & 9 & - & Solubility Prediction \\
\textsc{ogbg-molhiv} & 25.5 (avg) & 27.5 (avg) & 9 & 2 & Drug Discovery \\
\textsc{ogbg-moltox21} & 18.6 (avg) & 19.4 (avg) & 9 & 12 (multi-label) & Toxicology \\
\textsc{MNIST (Superpixels)} & 75 & 142 & 1 & 10 & Vision (Digits) \\
\textsc{CIFAR10 (Superpixels)} & 85 & 170 & 1 & 10 & Vision (Objects) \\
\textsc{ModelNet} & 100 (fixed) & 150 (fixed) & 3 & 40 & 3D Shape Classification \\
\textsc{Cuneiform} & 62 (avg) & 150 (avg) & 1 & 30 & Archaeology / OCR \\
\textsc{MSRC 21} & 212 (avg) & 336 (avg) & 4 & 21 & Image Segmentation \\
\bottomrule
\end{tabular}
\end{table*}

\subsection{Pre-training Source Datasets (A1)}
For pretraining \ourmethod, we use 10 diverse datasets covering molecular graphs, drugs, computer vision, and 3D shapes. The statistics for each dataset are summarized in~\Cref{tab:pretrain_statistics}. The detailed information is as follows:
\begin{itemize}
    \item \textbf{\textsc{zinc}}~\citep{dwivedi2023benchmarking} is a molecular property prediction dataset where the task is regressing the constrained solubility values of molecules. We report mean absolute error (MAE) as the evaluation metric.
    \item \textbf{\textsc{molhiv, molesol, moltox21}}~\citep{hu2020open} is a collection of  molecular graphs from the OGB benchmark covering drug discovery and toxicity prediction tasks. Depending on the dataset, we perform binary classification (\textsc{molhiv}), regression (\textsc{molesol}), or multi-label classification (\textsc{moltox21}). Performance is measured using ROC-AUC or RMSE, as appropriate.
    \item \textbf{\textsc{mnist, cifar10}}~\citep{dwivedi2023benchmarking} are computer vision datasets converted into graph-structured superpixels. Each image is modeled as a fixed-structure graph, with 1-dimensional input features and a 10-way classification objective.

    \item \textbf{\textsc{ModelNet}}~\citep{wu20153d} is a 3D object classification benchmark where shapes are represented as fixed-size point cloud graphs. We use the 10-class subset.
    \item \textbf{\textsc{Cuneiform}}~\cite{MorrisTUD} is a graph-based OCR dataset derived from ancient script symbols, consisting of 62-node graphs with 150 edges on average and a 30-class prediction target.
    \item \textbf{\textsc{MSRC-21}}~\cite{MorrisTUD} is an image segmentation dataset where region adjacency graphs are constructed from visual scenes. Each graph has approximately 212 nodes and 336 edges, with 4-dimensional node features and 21 semantic class labels.

\end{itemize}

\subsection{Transferability to Unseen Datasets and Input Features (A2)}
To evaluate the transferability of \ourmethod to unseen input features, we choose the following datasets summarized in \Cref{tab:a2} and explained below:

\begin{table*}[t]
\centering
\scriptsize
\caption{Statistics of finetuning datasets used in our experiments. For node classification datasets (citation networks), we report the total number of nodes and edges. For graph classification datasets (bioinformatics), we report the number of graphs and average graph sizes.}
\label{tab:a2}
\begin{tabular}{lccccc}
\toprule
\textbf{Dataset} & \textbf{\# Graphs / Nodes} & \textbf{\# Edges} & \textbf{\# Features} & \textbf{\# Classes} & \textbf{Domain / Task} \\
\midrule
\textsc{Cora} & 2,708 nodes & 5,429 & 1,433 & 7 & Citation Network / Node Classification \\
\textsc{Citeseer} & 3,327 nodes & 4,732 & 3,703 & 6 & Citation Network / Node Classification \\
\textsc{Pubmed} & 19,717 nodes & 44,338 & 500 & 3 & Citation Network / Node Classification \\
\textsc{MUTAG} & 188 graphs & 17.9 (avg) & 7 & 2 & Bioinformatics / Graph Classification \\
\textsc{PROTEINS} & 1,113 graphs & 39.1 (avg) & 3 & 2 & Bioinformatics / Graph Classification \\
\bottomrule
\end{tabular}
\end{table*}

\begin{itemize}
    \item \textbf{\textsc{Cora}, \textsc{Citeseer}, \textsc{Pubmed}}~\cite{yang2016revisiting}: In these datasets, nodes represent academic papers and edges denote citation links. Each node is assigned a class label corresponding to a subject area. The task is to predict the category of a paper based on its content features and citation graph. Models are evaluated under transductive learning settings using fixed splits~\cite{yang2016revisiting}.
    
    \item \textbf{\textsc{MUTAG}}~\cite{MorrisTUD}: A binary classification dataset of small molecule graphs. Nodes represent atoms with categorical features, and graphs are labeled based on mutagenic effect on a bacterium.
    
    \item \textbf{\textsc{PROTEINS}}~\cite{MorrisTUD}: A dataset of protein structures modeled as graphs where nodes represent secondary structure elements and edges reflect neighborhood in the amino acid sequence. Each graph is labeled as enzyme or non-enzyme.
\end{itemize}

\section{Implementation Details}
\label{app:implementation}

We implement \ourmethod using PyTorch~\citep{paszke2019pytorch} (BSD-3 Clause license) and PyTorch Geometric~\citep{fey2019fast} (MIT license). For experiment tracking and hyperparameter logging, we utilize the Weights and Biases framework~\citep{wandb}. Experiments were conducted with NVIDIA RTX A6000, RTX 4090, and NVIDIA A100 GPUs.

For all experiments, we use the GPS framework~\citep{rampavsek2022recipe} with the \textsc{GIN} message passing layer~\citep{xu2019how} for $\{\text{GNNLayer}^{(\ell, \mA)(\cdot, \mA)}\}_{\ell=0}^L$, and we use standard message passing layer for other operators.

\subsection{Pre-training on Different Source Datasets (Q1)}
\label{sec:app:rq1}
To evaluate large-scale transfer, we pretrain \ourmethod on a diverse set of 10 graph datasets spanning multiple domains, as described in \Cref{sec:app:data}. Each training epoch cycles through all datasets once, optimizing dataset-specific objectives. We train for 500 epochs and checkpoint every 25 epochs. Hyperparameters are detailed in \Cref{tab:hyperparams}. To accelerate training, 
\begin{enumerate*}[label=(\arabic*)]
    \item we use \texttt{DataParallel} to support multi-GPU runs,
    \item cache the random projection matrix $\mC$ and refresh every 100 steps,
    \item sample 10,000 graphs randomly at each epoch for \textsc{mnist} and \textsc{cifar10}, and
    \item sample 128 nodes with 6-nearest neighbors as edges for \textsc{ModelNet} in each graph.
\end{enumerate*}

% -------- Hyperparameter Table --------
\begin{table*}[t]
\centering
\scriptsize
\caption{Hyperparameter Configuration for Pretraining Stage.}
\label{tab:hyperparams}
\begin{tabular}{lc}
\toprule
\textbf{Category} & \textbf{Hyperparameter (Value)} \\
\midrule
\multicolumn{2}{c}{\textbf{Architecture}} \\
\midrule
Activation Function & \texttt{ReLU} \\
Attention Type in GPS & \texttt{PerformerAttention} \\
GPS Heads & \texttt{4} \\
Channels $h^{(\ell)}$ & \texttt{256} \\
Random Projection Dim $h$ & \texttt{512} \\
Backbone GNNLayer & \texttt{gps\_gine} \\
Number of Layers $L$ & \texttt{6} \\
Input PE Dim $h_s$ & \texttt{20} \\
Use Random Projections & \texttt{True} \\
\# Node-Covariance Operators $k$ & \texttt{0, 2} \\
\midrule
\multicolumn{2}{c}{\textbf{Training Setup}} \\
\midrule
Pretraining Epochs & \texttt{500} \\
Batch Size & \texttt{64} \\
Dropout & \texttt{0.0} \\
Learning Rate & \texttt{0.0001} \\
Weight Decay & \texttt{0.0} \\
Normalization Type & \texttt{batchnorm} \\
% Pretraining Datasets & \texttt{ZINC, ogbg-molesol, ogbg-molbace, ogbg-moltox21, ogbg-molhiv, CIFAR10, MNIST, MSRC\_21, Cuneiform, ModelNet} \\
\bottomrule
\end{tabular}
\end{table*}

\subsection{Evaluation on Unseen Datasets and Input Spaces (Q2)}
\label{sec:app:transfer}

To evaluate the transferability of \ourmethod to unseen datasets with novel input features, we freeze the pretrained encoder and evaluate its representations by training lightweight classifiers on new target datasets. These datasets span both node-level and graph-level classification tasks, with input feature spaces and labels disjoint from those used during pretraining.

For each target dataset, we instantiate a prediction head using one of the following:
\begin{enumerate*}[label=(\arabic*)]
    \item a \textbf{multi-layer perceptron (MLP)} for both node and graph classification tasks;
    \item a \textbf{2-layer GCN}~\cite{kipf2017semisupervised} applied to node classification benchmarks (\textsc{Cora, Citeseer, Pubmed}); and
    \item a \textbf{2-layer GIN}~\cite{xu2019how} for graph classification benchmarks (\textsc{MUTAG, PROTEINS}).
\end{enumerate*}  All prediction heads are trained with frozen \textsc{All-In} features as input. No gradients are backpropagated into the encoder during this stage.

For MLPs, we use a single hidden layer of size 128 with ReLU activation, followed by a softmax or sigmoid output layer, depending on whether the task is single-label or multi-label. We train all classifiers using the Adam optimizer with a learning rate of 0.001 and early stopping based on validation loss. Node classification models are trained on the standard 20/30/50 splits~\cite{yang2016revisiting} and evaluated using accuracy. For graph classification, we perform 10-fold stratified cross-validation and report the mean and standard deviation of classification accuracy.

All transfer experiments are implemented in PyTorch and PyTorch Geometric. Environment and optimization settings match those described in \Cref{sec:app:rq1}.

\end{document}